%% file: main.tex
\documentclass[nohyperref]{article}

\input{math_commands.tex}

\usepackage{hyperref}
\usepackage{url}

\usepackage{microtype}
\usepackage[pdftex]{graphicx}
\usepackage{booktabs} 
\usepackage[font=small]{caption}
\usepackage[font=small]{subcaption}
\usepackage{bm}
\usepackage{enumitem}
\usepackage[ruled,vlined]{algorithm2e}
\usepackage{setspace}
\usepackage{titlesec}
\usepackage{physics}
\usepackage{chngpage}
\usepackage{dsfont}
\usepackage{etoolbox,lineno}

\usepackage{lipsum}
\usepackage{wrapfig}
 
\usepackage[preprint]{icml2023}

\usepackage{amsmath}
\usepackage{amssymb}
\usepackage{mathtools}
\usepackage{amsthm}
\usepackage{multirow}
\usepackage{multicol}

\usepackage[capitalize,noabbrev]{cleveref}

\theoremstyle{plain}
\newtheorem{theorem}{Theorem}
\newtheorem{proposition}[theorem]{Proposition}
\newtheorem{lemma}[theorem]{Lemma}
\newtheorem{corollary}[theorem]{Corollary}

\newtheorem{remark}[theorem]{Remark}

\theoremstyle{definition}

\usepackage[textsize=tiny]{todonotes}

\newcommand{\breakcell}[2][c]{%
  \begin{tabular}[#1]{@{}c@{}}#2\end{tabular}}

\definecolor{cardinal}{rgb}{0.77, 0.12, 0.23}

\newcommand{\yuji}[1]{\textcolor{black}{#1}}
\newcommand{\revision}[1]{\textcolor{black}{#1}}
\newcommand{\revisionnew}[1]{\textcolor{black}{#1}}
\icmltitlerunning{Improving Fair Training under Correlation Shifts}

\begin{document}

\twocolumn[
\icmltitle{Improving Fair Training under Correlation Shifts}



\icmlsetsymbol{equal}{*}

\begin{icmlauthorlist}
\icmlauthor{Yuji Roh}{kaist}
\icmlauthor{Kangwook Lee}{wisconsin}
\icmlauthor{Steven Euijong Whang}{kaist}
\icmlauthor{Changho Suh}{kaist}
\end{icmlauthorlist}

\icmlaffiliation{kaist}{Department of Electrical Engineering, KAIST}
\icmlaffiliation{wisconsin}{Department of Electrical and Computer Engineering, University of Wisconsin-Madison}

\icmlcorrespondingauthor{Steven E. Whang}{swhang@kaist.ac.kr}

\icmlkeywords{Machine Learning, ICML}

\vskip 0.3in
]



\printAffiliationsAndNotice{}  

\begin{abstract}
Model fairness is an essential element for Trustworthy AI.
While many techniques for model fairness have been proposed, most of them assume that the training and deployment data distributions are identical, which is often not true in practice.
In particular, when the bias between labels and sensitive groups changes, the fairness of the trained model is directly influenced and can worsen.
We make two contributions for solving this problem. First, we analytically show that existing in-processing fair algorithms have fundamental limits in accuracy and \revision{group} fairness. We introduce the notion of \textit{correlation shifts}, which can \yuji{explicitly} capture the change of the above bias. 
Second, we propose a novel pre-processing step that samples the input data to reduce correlation shifts and thus enables the in-processing approaches to overcome their limitations. We formulate an optimization problem for adjusting the data ratio among labels and sensitive groups to reflect the shifted correlation.
A key \revision{benefit} of our approach lies in \textit{decoupling} the roles of pre- and in-processing approaches: correlation adjustment via pre-processing and unfairness mitigation on the processed data via in-processing.
Experiments show that our framework effectively improves existing in-processing fair algorithms w.r.t.\@ accuracy and fairness, both on synthetic and real datasets.

\end{abstract}

\section{Introduction}
\label{sec:intro}


Model fairness is becoming indispensable in many \revision{artificial intelligence (AI)} applications to prevent discrimination against specific groups such as gender, race, or age~\citep{DBLP:conf/kdd/FeldmanFMSV15, DBLP:conf/nips/HardtPNS16} \revision{or individuals~\citep{Dwork:2012:FTA:2090236.2090255}. In this work, we focus on group fairness, and}
there are three prominent \revision{group} fairness approaches: pre-processing, where training data is debiased; in-processing, where model training is tailored for fairness; and post-processing, where the trained model's output is modified to satisfy fairness -- see more related works discussed in Sec.~\ref{sec:relatedwork}.

While fairness in-processing approaches are commonly used to mitigate unfairness, most of them make the limiting assumption that the training and deployment data distributions are the same~\citep{DBLP:conf/aistats/ZafarVGG17, DBLP:conf/aies/ZhangLM18, roh2021fairbatch}.
However, the two distributions are usually different, especially in terms of data biases~\citep{NEURIPS2019_373e4c5d, maity2021does}.
For example, a recent work shows that the bias amounts likely differ between previously collected data and recently collected data~\citep{ding2021retiring}.
\yuji{Moreover, when the data bias changes, the fairness and accuracy of the trained model are now unpredictable at deployment, as the above assumption is broken.}

\yuji{In this work, we introduce the notion of \textit{correlation shifts} between the label $\ry$ and group attribute $\rz$ in the data to systematically address the data bias changes. Although several works have been recently proposed to investigate fair training on different types of distribution shifts, including covariate and concept shifts~\citep{singh2021fairness, mishler2022fair}, they usually do not explicitly consider bias changes between $\ry$ and $\rz$. In comparison, our correlation shift enables us to theoretically analyze how exactly bias changes affect fair training -- see how correlation shift compares with other types of distribution shifts in Sec.~\ref{sec:relatedwork}.}



\yuji{For fair training under correlation shifts, we first 1) analyze the fundamental accuracy and fairness limits of in-processing approaches with the fixed distribution assumption using the notion of correlation in the data and then 2) design a novel pre-processing step to boost the performances of in-processing approaches under the correlation shifts.}
\yuji{We show} that existing in-processing fair algorithms are indeed limited by the training distribution and may perform poorly on the deployment distribution. 
In particular, a high ($\ry$, $\rz$)-correlation results in a poor accuracy-fairness tradeoff for any fair training.
Therefore, as most in-processing fair algorithms assume identical training and deployment distributions, there is no guarantee their performances on the training data carry over to the deployment data. 
Based on the theoretical analysis, we propose a pre-processing step for reducing the shifted correlation by taking samples of ($\ry, \rz$)-classes. 
\revision{In a streaming setup, one can estimate a possible range of shifted correlations. 
Using this range,} we solve an optimization problem that finds the new data ratio among ($\ry, \rz$)-classes to adjust the correlation for the shift. 
\revision{The new sampled data is then used as the input of any in-processing approach, giving it a better chance to perform well.}

\begin{figure}[t]
\centering
\vspace{-0.05cm}
\includegraphics[width=0.9\columnwidth,trim=0cm 0.3cm 0cm 0cm]{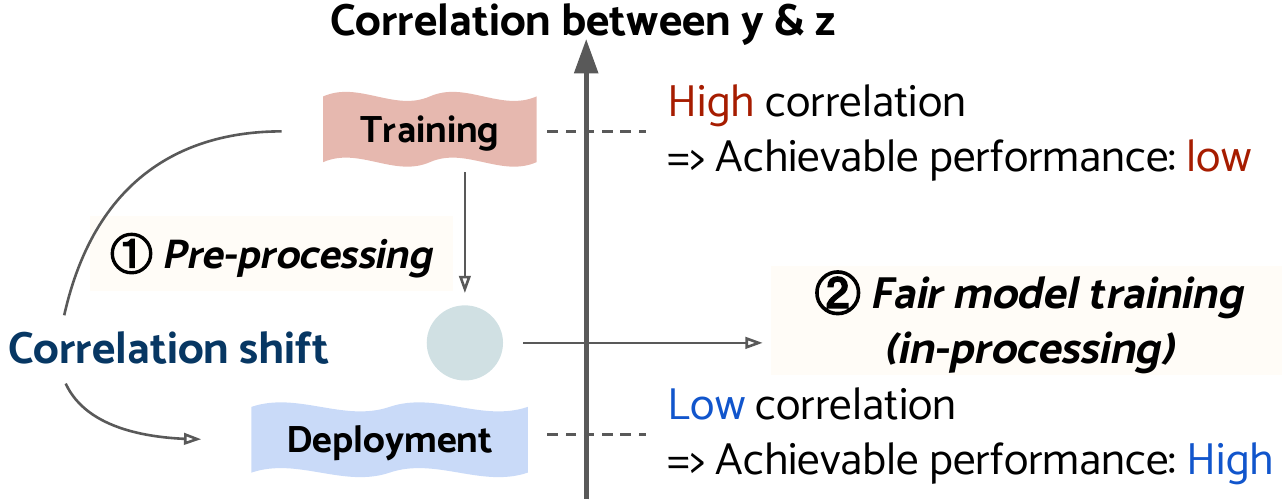}
\vspace{-0.3cm}
\caption{A high-level workflow under correlation shifts. The central axis represents the correlation between the label $\ry$ and sensitive group attribute $\rz$. The correlation of training data is usually higher than that of deployment data. In Sec.~\ref{sec:correlationshift}, we show that the correlation determines the achievable performance of fair training.
Thus, we first run our pre-processing framework and then apply existing fair algorithms on the processed data to address the correlation shift and improve the performances of fair training.}
\vspace{-0.45cm}
\label{fig:overall}
\end{figure}

\textcolor{black}{A key advantage} of our framework is \textit{the decoupling of pre-processing and in-processing for unfairness mitigation} where the pre-processing adjusts the correlation while the in-processing performs the rest of the unfairness mitigation, as described in Figure~\ref{fig:overall}.
We note that our pre-processing aims to boost the performances of in-processing approaches based on our theoretical analysis, whereas existing pre-processing approaches for fairness simply remove the bias in the data and are not designed to explicitly benefit the in-processing approaches -- see Sec.~\ref{sec:accfair} for details. 
Our framework thus takes the best of both worlds of \textcolor{black}{pre- and in-processings} where (1) pre-processing solves the data problems, and (2) in-processing performs its best on the improved data.

In our experiments, we verify our theoretical results and demonstrate how our framework outperforms state-of-the-art pre-processing and in-processing baselines.
Experiments on both synthetic and real-world datasets show that our framework effectively improves the accuracy and fairness performances of the state-of-the-art in-processing approaches~\citep{DBLP:conf/aistats/ZafarVGG17, DBLP:conf/aies/ZhangLM18, roh2021fairbatch} under correlation shifts. Also, our framework performs better than two-step baselines that first run an existing pre-processing approach~\citep{DBLP:journals/kais/KamiranC11} and then an in-processing approach.
\revision{We also show how to reasonably estimate correlation shifts and that our framework is robust against incorrect estimates.}

\vspace{-0.35cm}
\paragraph{Summary of Contributions} \textbf{(1)} We introduce the notion of correlation shifts, \revision{which is key to improving fair training against data bias changes.} \textbf{(2)} Using the notion of correlation, we theoretically show that existing in-processing fair algorithms are limited by the training distribution and may perform poorly on the deployment distribution. \textbf{(3)} \revision{Based on the theoretical analysis,} we propose a novel pre-processing step to boost the performances of fair in-processing approaches. \textbf{(4)} We demonstrate that our framework effectively improves the performances of the state-of-the-art fair algorithms under correlation shifts, \revision{even if the correlation estimates are inaccurate.}


\vspace{-0.3cm}
\paragraph{\yuji{Notation \& Fairness Metrics}} 
Let $\theta$ be the model weights, $\rx \in \sX$ be the input feature to the model, $\ry \in \sY$ be the true label, and $\hat{\ry} \in \sY$ be the predicted label where $\hat{\ry}$ is a function of ($\rx$, $\theta$). Let $\rz \in \sZ$ be a sensitive group attribute, e.g., gender or race. We assume a binary setting ($\sY = \sZ = \{0,1\}$).
We focus on the prominent group fairness metrics, demographic parity (DP)~\citep{DBLP:conf/kdd/FeldmanFMSV15} and equalized odds (EO)~\citep{DBLP:conf/nips/HardtPNS16}, where DP is achieved when the positive prediction rates are the same for the groups (i.e., $\Pr(\hat{\ry}{=}1|\rz{=}1) {=} \Pr(\hat{\ry}{=}1|\rz{=}0)$) and EO is achieved when the label-wise accuracies are the same for the groups (i.e., $\Pr(\hat{\ry}{=}y|\ry{=}y, \rz{=}1) {=} \Pr(\hat{\ry}{=}y|\ry{=}y, \rz{=}0), \forall y \in \{0,1\}$). 

\vspace{-0.1cm}
\section{\resizebox{0.95\hsize}{!}{Limitations of Prior Work \& Problem Setup}}
\label{sec:limitations}
\vspace{-0.1cm}

Most fair in-processing approaches assume that the training and deployment distributions are the same, which means that they assume the same level of bias as well.
However, data bias may shift over time as confirmed by recent studies~\citep{NEURIPS2019_373e4c5d, maity2021does, ding2021retiring}, which means that the deployment data may actually have a different bias than the training data. 


A bias change in the deployment data may have an adverse effect on a trained model's performance. \revisionnew{For example, suppose we train a perfectly fair model on the training data while sacrificing accuracy. Here, if the bias between the label and group changes in the deployment distribution, the trained model's fairness can worsen while the accuracy still remains imperfect. The underlying problem is that the model was trained with a different bias in mind -- see a detailed toy example to show this limitation in Sec.~\ref{appendix:toyexample}. Thus, we aim to overcome this issue under the following setup.}

\vspace{-0.3cm}
\paragraph{Problem Setup}
\revisionnew{Let $D_\text{train}$ and $D_\text{deploy}$ be the training and deployment data, respectively. We assume access to the ($\rx$, $\ry$, $\rz$) values in $D_\text{train}$ and $m$ samples of the ($\ry$, $\rz$) values in $D_\text{deploy}$. Accessing some of the new data is a common and practical scenario for detecting drifts~\citep{lu2018learning} in a data stream. Here, we can detect whether there exists a bias change (shift) by comparing the ($\ry$, $\rz$) values in $D_\text{train}$ and a small number of ($\ry$, $\rz$) values in $D_\text{deploy}$. Our goal is to have high fairness and accuracy even when the bias shifts.}

\revisionnew{To this end, we first \textbf{1)} theoretically analyze why the bias change affects the performances of fair training by introducing the notion of correlation (\textbf{Sec.~\ref{sec:correlationshift}}) and then \textbf{2)} design a pre-processing step to address such bias changes and boost the performances of any fair in-processing approaches under the above problem setup (\textbf{Sec.~\ref{sec:framework}}).}

\section{Fair Training under Correlation Shifts}
\label{sec:correlationshift}
\vspace{-0.1cm}

\yuji{To systematically study the effects of data bias changes, we first analyze the achievable performances of fair training via a {\em correlation between y and z that can be used to explicitly measure data bias regarding sensitive groups} (Sec.~\ref{sec:fundamentallimits}) and discuss the limitations of fair in-processings when this correlation shifts (Sec.~\ref{sec:shift_effects}).
We define correlation as follows:}


\vspace{-0.3cm}
\paragraph{($\ry, \rz$)-correlation}

Data bias can be represented via the correlation between the label $\ry$ and the sensitive group attribute $\rz$, where the correlation represents a statistical relationship between two random variables. We thus define {\em ($\ry, \rz$)-correlation} using Pearson's correlation coefficient~\citep{lee1988thirteen} $\rho_{\ry\rz} = \frac{Cov(\ry, \rz)}{\sigma(\ry)\sigma(\rz)}$, \textcolor{black}{which is known to effectively capture biases in real-world scenarios. Here,} $Cov(\cdot)$ is the covariance, and $\sigma(\cdot)$ is the standard deviation.
As we assume $\ry$ and $\rz$ are binary, we can express $\rho_{\ry\rz}$ as follows~\citep{cohen1975applied}:\vspace{-0.2cm}\\
\resizebox{\linewidth}{!}{
\centering
\begin{minipage}{\linewidth}
\begin{align*}
    \rho_{\ry\rz} = 
    \frac{\Pr(\ry=1, \rz=1)\Pr(\ry=0, \rz=0) - \Pr(\ry=1, \rz=0)\Pr(\ry=0, \rz=1)}{\sqrt{\Pr(\ry=1)\Pr(\ry=0)\Pr(\rz=1)\Pr(\rz=0)}}.
\end{align*}
\end{minipage}
}

\revisionnew{This notion of correlation plays a key role in both analyzing the achievable performances of fair training (Sec.~\ref{sec:correlationshift}) and improving fair training against data bias changes (Sec.~\ref{sec:framework}).}

\vspace{-0.1cm}
\subsection{Achievable Performance Analysis via Correlation}
\label{sec:fundamentallimits}
\vspace{-0.1cm}

We identify the fundamental accuracy and fairness limits of fair training based on the ($\ry$, $\rz$)-correlation of the data. 
We first analyze the most common case of improving the fairness w.r.t. a single fairness metric. We then extend our analysis to the more complicated case of improving the fairness w.r.t. multiple metrics, which is important for capturing various social contexts, but has been seldom studied in the literature. 
In both cases, we show that the ($\ry$, $\rz$)-correlation determines the achievable performance of fair training.

\vspace{-0.15cm}
\paragraph{CASE 1 -- Fair Training w.r.t.\@ a Single Metric}
When improving group fairness, fair training is known to face an accuracy-fairness tradeoff, where the accuracy is sacrificed to make the model fairer. 
Recently, \citet{pmlr-v81-menon18a} investigate that the accuracy-fairness tradeoff w.r.t. demographic parity is affected by how much $\ry$ and $\rz$ are \textit{aligned} in the data. For example, if the $\ry$ and $\rz$ values are identical for all examples, achieving high fairness may require low accuracy. In contrast, if $\ry$ and $\rz$ are randomly set, fairness and accuracy can be achieved together. 
The following proposition shows this previous work's result.

\vspace{0.05cm}
\begin{proposition}[From \citet{pmlr-v81-menon18a}]
(Informal) When a model is trained w.r.t. demographic parity, a high alignment between $\ry$ and $\rz$ leads to a worse accuracy-fairness tradeoff.\label{prop:menon}
\end{proposition}
\vspace{-0.15cm}

For our purposes, we infer a similar relationship using ($\ry, \rz$)-correlation based on Proposition~\ref{prop:menon}. The following lemma makes a connection between the ($\ry$, $\rz$)-correlation and the conditional probabilities of $\ry$ given $\rz$ under some conditions.
\begin{lemma}
If the marginal probabilities of $\ry$ and $\rz$ (i.e., $\Pr(\ry=y)$ and $\Pr(\rz=z)$) remain the same, the ($\ry, \rz$)-correlation $\rho_{\ry\rz}$ is proportional to the difference between the conditional probabilities of $\ry$ given different $\rz$ values (i.e., $\Pr(\ry=1|\rz=1) - \Pr(\ry=1|\rz=0)$).\label{lemma:proportional}
\end{lemma}
\vspace{-0.15cm}

The proof for Lemma~\ref{lemma:proportional} can be found in Sec.~\ref{appendix:correlationproperty}. 
By applying Lemma~\ref{lemma:proportional} to Proposition~\ref{prop:menon}, we can infer that the ($\ry, \rz$)-correlation also determines the accuracy-fairness tradeoff under certain conditions. 
According to the previous work~\citep{pmlr-v81-menon18a}, alignment is defined as how many examples in each group have a specific label. We can thus measure alignment as $\Pr(\ry=y|\rz=1)+\Pr(\ry=y'|\rz=0)$ where $y \neq y'$ and $y, y' \in \{0,1\}$. Then, we can convert this term into the difference between the conditional probabilities of $\ry$ given different $\rz$ values, which is used in Lemma~\ref{lemma:proportional} -- see details in Sec.~\ref{appendix:singlemetric}.
As a result, we derive the following corollary, which shows that the ($\ry, \rz$)-correlation determines the achievable accuracy-fairness tradeoff.
\begin{corollary}[]
\label{corollary:dp}
When a model is trained w.r.t. demographic parity, and the marginal probabilities of $\ry$ and $\rz$ remain the same, the achievable accuracy-fairness tradeoff of the model is determined by the ($\ry, \rz$)-correlation. The higher the correlation, the worse the accuracy-fairness tradeoff.\label{eq:corollary_single}
\end{corollary}

\begin{remark}\label{remark:conditions}
\textcolor{black}{We can relax the assumption that marginal probabilities are fixed in Corollary~\ref{corollary:dp}. 
When the marginal probabilities of $\ry$ and $\rz$ change up to $\gamma_\ry$ and $\gamma_\rz$, respectively, 
Corollary~\ref{corollary:dp} can be generalized as follows: the achievable accuracy-fairness tradeoff is determined by $\rho_{\ry, \rz} \cdot \eta$, where $\rho_{\ry, \rz}$ is the ($\ry, \rz$)-correlation and $\eta \in [\sqrt{\tfrac{\Pr(\ry=1)-\gamma_\ry - (\Pr(\ry=1)+\gamma_\ry)^2}{\Pr(\rz=1)+\gamma_\rz - (\Pr(\rz=1)-\gamma_\rz)^2}}, \sqrt{\tfrac{\Pr(\ry=1)+\gamma_\ry - (\Pr(\ry=1)-\gamma_\ry)^2}{\Pr(\rz=1)-\gamma_\rz - (\Pr(\rz=1)+\gamma_\rz)^2}}]$. 
Thus, the higher the $\rho_{\ry, \rz} \cdot \eta$, the worse the tradeoff.
In our framework, we actually allow the marginal probabilities of $\ry$ and $\rz$ to change up to $\gamma_\ry$ and $\gamma_\rz$ -- see details in Sec.~\ref{sec:optimization}.}
\end{remark}

\vspace{-0.25cm}
\paragraph{Simulation}
We now confirm our theoretical \revision{analysis} via a simulation -- see details on the setting in Sec.~\ref{appendix:simulation_setting}.
The left plot in Figure~\ref{fig:simulation} shows the accuracy-unfairness performances of classifiers on two synthetic datasets with low and high correlations when using DP for measuring fairness (the plot on the right will be explained later). 
We measure unfairness where a lower value indicates better fairness (i.e., perfectly fair when 0)
-- see the exact metrics in Sec.~\ref{sec:experiments}.
We generate various synthetic classifiers on the two datasets to show the full range of possible model performances. The blue dots (red crosses) are the classifiers on the dataset with low (high) correlation. 
As a result, low correlation results in better accuracy-fairness tradeoffs (i.e., close to the bottom right). We will discuss the black stars and squares in Sec.~\ref{sec:shift_effects}.

\begin{figure}[t]
\centering
\includegraphics[width=0.98\columnwidth,trim=0cm 0.3cm 0cm 0cm]{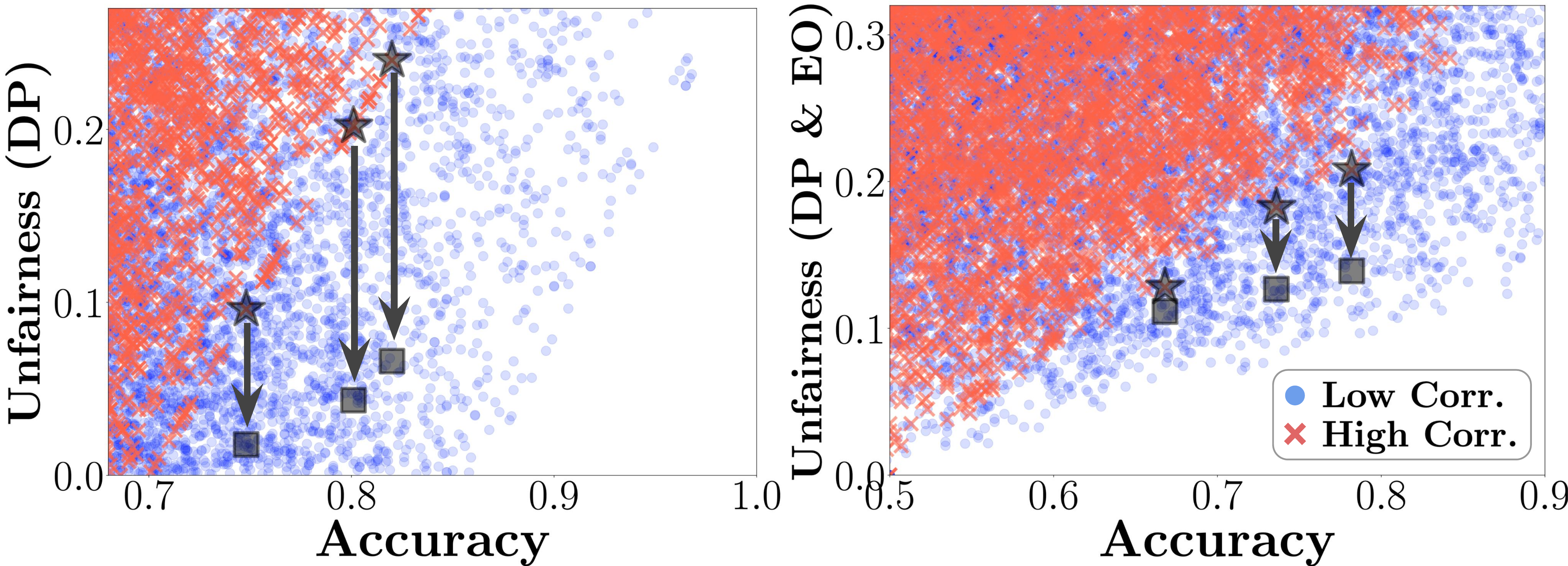}
\vspace{-0.4cm}
\caption{\small Simulation results of accuracy-unfairness (left: DP and right: DP \& EO) performances of various classifiers on two synthetic datasets that have low and high ($\ry$, $\rz$)-correlations. 
\revision{Each blue dot (red cross) indicates a single classifier's performance on the low (high) correlation data.} 
The lower the unfairness value, the better. For each dataset, we generate enough classifiers to show the full range of possible performances. 
As a result, low correlation enables classifiers to attain better accuracy-fairness tradeoffs (i.e., close to the bottom right), regardless of supporting single or multiple fairness metrics.
Also, the optimal classifiers trained on high correlation data (black stars) have suboptimal performances on the low correlation data (black squares).}
\vspace{-0.3cm}
\label{fig:simulation}
\end{figure}

\begin{remark}\label{remark:single_eo}
We note that Corollary~\ref{eq:corollary_single} does not necessarily apply for equalized odds (EO), which is achieved when the accuracies conditioned on the true label are the same for the groups.
In theory, a perfect classifier can achieve perfect fairness w.r.t. EO, so the ($\ry$, $\rz$)-correlation does not determine the limits of the accuracy-fairness tradeoff w.r.t. EO. However, classifiers are not perfect in practice, and we empirically observe that higher ($\ry$, $\rz$)-correlation leads to a worse accuracy-fairness tradeoff w.r.t. EO as well -- see the empirical results in Sec.~\ref{appendix:eo_empirical}.
\end{remark}

\vspace{-0.25cm}
\paragraph{\resizebox{0.84\hsize}{!}{CASE 2 -- Fair Training w.r.t.\@ Multiple Metrics}}
\hspace{-0.2cm}
Beyond the single metric \revision{case}, we now extend our analysis to support multiple fairness metrics together. 
Although supporting multiple metrics is necessary to address fairness under various social contexts, most fairness works do not address this problem. A major challenge is that group fairness metrics are known to be mutually exclusive, which means that the metrics cannot be perfectly satisfied together, unless the data is completely unbiased~\citep{barocas-hardt-narayanan, DBLP:conf/innovations/KleinbergMR17}.
Interestingly, our ($\ry, \rz$)-correlation provides an opportunity to support multiple metrics as much as possible in a principled fashion.
We focus on supporting two metrics and leave the support of more metrics as future work.

We first analyze when a model can improve both DP and EO to get hints on the relationship between the ($\ry, \rz$)-correlation and improving fairness w.r.t. two metrics.
Here, we focus on the $\varepsilon$-fairness of both metrics, which is a relaxed version of perfect fairness~\citep{barocas-hardt-narayanan}, where $\varepsilon$ indicates the unfairness level of the model. For example, we can define $\varepsilon$-DP as $|\Pr(\hat{\ry}{=}1|\rz{=}1) {-} \Pr(\hat{\ry}{=}1|\rz{=}0)| \leq \varepsilon$ and $\varepsilon$-EO as \resizebox{0.81\hsize}{!}{$|\Pr(\hat{\ry}{=}y|\ry{=}y, \rz{=}1) {-} \Pr(\hat{\ry}{=}y|\ry{=}y, \rz{=}0)| \leq \varepsilon$}. Here, $\varepsilon \in [0,1]$ and a lower $\varepsilon$ means higher fairness.
The \revision{next} propo-sition shows when a model can achieve both $\varepsilon$-DP and $\varepsilon$-EO.

\begin{proposition}[$\varepsilon$-DP \& $\varepsilon$-EO]\label{prop:dpeo}
Let a model achieve both $\varepsilon$-DP and $\varepsilon$-EO. Then, the following inequality holds:
$\{ |\Pr(\ry=y|\rz=z)-\Pr(\ry=y|\rz=z')| ~\cdot$ $|\Pr(\hat{\ry}=y|\ry=y,\rz=z)-\Pr(\hat{\ry}=y|\ry=y',\rz=z)| \} \leq 2\varepsilon,~ y \neq y', z \neq z', y, y', z, z' \in \{0,1\}$.
\end{proposition}
\vspace{-0.1cm}

The proof for Proposition~\ref{prop:dpeo} can be found in Sec.~\ref{appendix:dp_eo}.
Here, the achievable fairness level $\varepsilon$ is affected by the data bias, which is represented by the difference between the conditional probabilities of $\ry$ given $\rz$ (i.e., $|\Pr(\ry=y|\rz=z)-\Pr(\ry=y|\rz=z')|$). 
For example, if the data is highly biased, a model cannot achieve high fairness (i.e., low $\varepsilon$) unless the second term on the left side (i.e., $|\Pr(\hat{\ry}=y|\ry=y',\rz=z)-\Pr(\hat{\ry}=y|\ry=y,\rz=z)|$) decreases. However, the second term is related to the accuracy, where lowering it to zero causes the model to make random predictions or predictions that are all the same. 
Thus, Lemma~\ref{lemma:proportional} and Proposition~\ref{prop:dpeo} give the following corollary.

\begin{corollary}[]
When a model is trained w.r.t. both DP and EO, and the marginal probabilities of $\ry$ and $\rz$ do not change, the achievable fairness of the model is determined by the ($\ry, \rz$)-correlation if the accuracy does not change. The higher the correlation, the lower the achievable fairness.\label{eq:corollary_dpeo}
\end{corollary}

Similar to the single-metric case, we confirm the theoretical result of supporting multiple metrics via the simulation in the right plot in Figure~\ref{fig:simulation}. Again, we observe that low correlation enables classifiers to attain better accuracy-fairness tradeoffs. 
We also leave discussions for improving fairness w.r.t.\@ another prominent metric called predictive parity (PP)~\citep{berk2021fairness} together with DP and EO in Secs.~\ref{appendix:pp_dp} ($\varepsilon$-PP and $\varepsilon$-DP) and~\ref{appendix:eo_pp} ($\varepsilon$-EO and $\varepsilon$-PP).

\vspace{-0.1cm}
\subsection{Effects of Correlation Shifts}
\label{sec:shift_effects}
\vspace{-0.1cm}

As the performance ranges of fair training are different according to the data correlation, a fair classifier learned on some training data does not necessarily have optimal performance on deployment data with a different correlation.
We say there is a \textit{correlation shift} between two data distributions $D_1$ and $D_2$ when $|\rho^{D_1}_{\ry\rz} - \rho^{D_2}_{\ry\rz}| \neq 0$, where $\rho^{D_i}_{\ry\rz}$ indicates the ($\ry, \rz$)-correlation of $D_i$. 
Here, the correlation \textit{shift} reflects the bias \textit{change} between data distributions. 

We revisit our previous simulation in Figure~\ref{fig:simulation} that helps to understand how the correlation shifts affect fair training.
We consider a correlation shift from a training distribution with high correlation to a deployment distribution with low correlation. 
We start from a few optimal classifiers trained on the high-correlated data, which are denoted as black stars. 
These classifiers unfortunately have suboptimal performances on the low-correlation deployment data (black squares).
Hence, in-processing approaches that assume the same training and deployment data distributions lose the chance of achieving better fairness and accuracy in the deployment distribution under correlation shifts.

Based on our theoretical and simulation results, we design a pre-processing approach using ($\ry, \rz$)-correlation as a tuning knob to give existing in-processing approaches a better opportunity to achieve maximum accuracy and fairness performances in the following section.

\section{Framework}
\label{sec:framework}
\vspace{-0.15cm}

To improve the performance of fair training in the presence of correlation shifts, we propose a novel fair training framework that consists of 
1) applying a pre-processing approach for reflecting the shifted correlation and 2) utilizing any existing in-processing algorithms for fair training on top of the improved data.
\revisionnew{In Sec.~\ref{sec:optimization}, we explain the pre-processing step, which estimates the correlation range and solves an optimization problem for finding a new data distribution that follows the shifted correlation by adjusting the ($\ry, \rz$)-class ratios.}
In Sec.~\ref{sec:algorithm}, we explain our overall training process and extensions to support similar distributions.

\vspace{-0.2cm}
\subsection{\revisionnew{Pre-processing for Correlation Shifts}}
\label{sec:optimization}
\vspace{-0.1cm}

\paragraph{Conditions} We design an optimization that finds the best ($\ry, \rz$)-class ratios given a ($\ry, \rz$)-correlation $\rho_{\ry\rz}$ by using Lemma~\ref{lemma:proportional}, which shows that $\rho_{\ry\rz}$ is proportional to the conditional probability difference under some assumptions. Note that the conditional probability can be written using the class weights.
Let the original data ratio of each ($\ry$=$y$, $\rz$=$z$)-class be $w_{\ry=y, \rz=y}$.
Let the new data ratio be $w'_{\ry=y, \rz=y}$, which is required to satisfy the shifted correlation of the deployment data.
Note that ${\sum}_{\forall y, z} w'_{\ry=y, \rz=z} {=} 1$. 
Also, let $c$ be the correlation constant that is the difference between the conditional probabilities of $\ry$ given $\rz$ (i.e., $\Pr(\ry{=}1|\rz{=}1) - \Pr(\ry{=}1|\rz{=}0)$) in the deployment data.
Let $[\alpha, \beta]$ be the range of the correlation constant $c$ (i.e., $c \in [\alpha, \beta]$). 
We use a relaxed version of the assumption in Lemma~\ref{lemma:proportional} where the marginal probabilities of $\ry$ and $\rz$ can change by up to $\gamma_\ry$ and $\gamma_\rz$\textcolor{black}{, as discussed in Remark~\ref{remark:conditions}}. 

\vspace{-0.3cm}
\paragraph{\revisionnew{Part 1: Estimating Correlation Range}} 
\revisionnew{We first find a shift range $[\alpha, \beta]$, which can be estimated or given in advance. When the estimation is needed, we consider the scenario where accessing some of the new data is possible, which is necessary for detecting data drifts~\citep{lu2018learning}.
As stated in Sec.~\ref{sec:limitations}, we assume we have the ($\ry, \rz$) values of $m$ samples in the deployment data.
We can estimate the new correlation constant $c$ by calculating the empirical probability $\hat{\Pr}(\ry|\rz)$ using the given $m$ samples. Here, we can use the maximum likelihood estimator, i.e., $\hat{\Pr}(\ry{=}y|\rz{=}z) {=}$ (\text{num. samples in ($y,z$)})/(\text{num. samples in $z$}). We can also compute a confidence interval of the estimated $c$ using standard concentration inequalities (e.g., Hoeffding's inequality) -- see a formal estimation guarantee and proof in Sec.~\ref{appendix:estimate_corr}.}

\vspace{-0.05cm}
\revisionnew{Some key questions are 1) whether the correlation can be estimated for small amounts of new data and 2) whether our framework works well even when the estimation is incorrect. We answer them via empirical results in Secs.~\ref{sec:realworld} and~\ref{sec:misspecifying_corr}.}

\vspace{-0.25cm}
\paragraph{\revisionnew{Part 2: Optimization}}\hspace{-0.3cm} With the above conditions \revisionnew{and estimated correlation}, we set a problem whose goal is to minimize the squared difference between the original and new data ratios to reduce the information loss, similar to other preprocessings~\citep{DBLP:conf/icml/ZemelWSPD13, quadrianto2019discovering}:\vspace{-0.35cm}\\
\resizebox{0.92\linewidth}{!}{
\hspace{0.5cm}
\centering
\begin{minipage}{\linewidth}
\vspace{-0.3cm}
\begin{gather*}
    \underset{w'}{\min} ~~ \scalebox{1.2}{$\sum_{\forall y, z}$} (w_{\ry=y, \rz=z}-w'_{\ry=y, \rz=z})^2 \vspace{-0.05cm}\\
    \text{s.t. }~~ \textcolor{black}{\alpha \leq} \frac{w'_{\ry=1, \rz=1}}{w'_{\ry=1, \rz=1} + w'_{\ry=0, \rz=1}} - \frac{w'_{\ry=1, \rz=0}}{w'_{\ry=1, \rz=0} + w'_{\ry=0, \rz=0}} \textcolor{black}{\leq \beta},\\
    |(w'_{\ry=1, \rz=1} + w'_{\ry=1, \rz=0})-{\Pr}_{\text{train}}(\ry=1)| \leq \gamma_\ry,\\
    |(w'_{\ry=1, \rz=1} + w'_{\ry=0, \rz=1})-{\Pr}_{\text{train}}(\rz=1)| \leq \gamma_\rz,\\
\hspace{-0.6cm}\scalebox{1.2}{$\sum_{\forall y, z}$} w'_{\ry=y, \rz=z} = 1,~~
    0 \leq w'_{\ry=y, \rz=z} \leq 1,~~\forall y \in \{0, 1\}, z \in \{0, 1\}
\end{gather*}
\end{minipage}
}\vspace{0.05cm}\\
where ${\Pr}_{\text{train}}(\ry{=}1) {=} w_{\ry=1, \rz=1} {+} w_{\ry=1, \rz=0}$ and ${\Pr}_{\text{train}}(\rz{=}1)$ ${=} w_{\ry=1, \rz=1} {+} w_{\ry=0, \rz=1}$. Note that if we know the exact shifted correlation value (i.e., $\alpha {=} \beta {=} c$), the first constraint \revisionnew{becomes an equality constraint.}

\vspace{-0.05cm}
The above optimization is a nonconvex quadratically constrained quadratic problem (QCQP), as the objective is quadratic, and the first constraint is nonconvex quadratic. However, the nonconvex QCQP is known to be hard to solve~\citep{d2003relaxations}. Thus, we apply the semidefinite (SDP) relaxation, which is one of the convex relaxations known to give a reasonable lower bound of the optimal value of the original problem~\citep{park2017general}:\vspace{-0.35cm}\\ 
\resizebox{0.92\linewidth}{!}{
\hspace{0.5cm}
\centering
\begin{minipage}{\linewidth}
\begin{gather*}
    \underset{X, x}{\min}~~ \textbf{Tr}(XP_0)+q_0^Tx \vspace{-0.05cm}\\
    \text{s.t.}~~  
    \textcolor{black}{\textbf{Tr}(XP_{\alpha}) \geq 0,~~ \textbf{Tr}(XP_{\beta}) \leq 0},\\
     |q_2^T x - {\Pr}_{\text{train}}(y=1)| \leq \gamma_\ry, ~~
     |q_3^T x - {\Pr}_{\text{train}}(z=1)| \leq \gamma_\rz, \\
     q_4^Tx = 1,~~ 0 \leq x_i \leq 1~~\forall i, ~~
     \begin{bmatrix} X & x \\ x^T & 1 \end{bmatrix} \succeq 0
\end{gather*}
\end{minipage}
}\vspace{-0.1cm}\\
where $\textbf{Tr}(\cdot)$ is the trace, \resizebox{0.59\hsize}{!}{$~X=xx^T$, $x = \begin{bmatrix} x_1 & x_2 & x_3 & x_4 \end{bmatrix}^T$}
\resizebox{1.05\hsize}{!}{${=}\begin{bmatrix} w'_{1, 1} \hspace{-0.1cm} & w'_{1, 0} \hspace{-0.1cm} & w'_{0, 1} \hspace{-0.1cm} & w'_{0, 0} \end{bmatrix}^T$,~ $q_0 = {-}2\begin{bmatrix} w_{1, 1} \hspace{-0.1cm} & w_{1, 0} \hspace{-0.1cm} & w_{0, 1} \hspace{-0.1cm} & w_{0, 0} \end{bmatrix}^T$,}
\resizebox{1.05\hsize}{!}{$q_2 {=}\begin{bmatrix} 1 \hspace{-0.1cm} & 0 \hspace{-0.1cm} & 1 \hspace{-0.1cm} & 0 \end{bmatrix}^T$, $q_3 {=} \begin{bmatrix} 1 \hspace{-0.1cm} & 1 \hspace{-0.1cm} & 0 \hspace{-0.1cm} & 0 \end{bmatrix}^T$, $q_4 {=} \begin{bmatrix} 1 \hspace{-0.1cm} & 1 \hspace{-0.1cm} & 1 \hspace{-0.1cm} & 1 \end{bmatrix}^T$, $P_0 {=} \diag(\mathbf{1})~$,} and \resizebox{0.60\hsize}{!}{$P_\gamma = \begin{bsmallmatrix} 0 & {-\gamma}/{2} & 0 & {(1-\gamma)}/{2}\\ {-\gamma}/{2} & 0 & {(-1-\gamma)}/{2} & 0\\ 0 & {(-1-\gamma)}/{2} & 0 & {-\gamma}/{2}\\ {(1-\gamma)}/{2} & 0 & {-\gamma}/{2} & 0 \end{bsmallmatrix}$} where $\gamma \in \{\alpha, \beta\}$. Conversion details are in Sec.~\ref{appendix:sdprelaxation}.
As this relaxed problem is now convex, we can solve it using convex optimization solvers (e.g., CVXPY~\citep{diamond2016cvxpy}).

\vspace{-0.05cm}
\begin{remark}
\revision{In distribution shift studies, it is common to assume that the training and deployment distributions differ only in specific aspects (e.g., only label dist.\@ is changed). Without such an assumption, the training and deployment distributions can be arbitrarily far apart, and there is no way to infer the new distribution via the training samples~\citep{huang2006correcting}. 
In our framework, we focus on the distribution changes of $\ry$ and $\rz$ and do not explicitly model the change of the input feature $\rx$ distribution. 
Our framework thus works best when the $\rx$ distribution does not change, but does not strictly require this condition unlike other previous works on fairness under different types of shifts~\citep{maity2021does,giguere2022fairness}.
We empirically show that our framework indeed performs well in a real-world scenario where the $\rx$ distribution does shift -- see Sec.~\ref{sec:realworld}. }
\end{remark}

\vspace{-0.3cm}
\subsection{Overall Training}
\label{sec:algorithm}
\vspace{-0.0cm}

\setlength{\textfloatsep}{0pt}
\vspace{-0.15cm}
\begin{algorithm}[h]
\small
\DontPrintSemicolon
\setstretch{1.0}
    \SetKwInput{Input}{Input}
    \SetKwInOut{Output}{Output}
    \SetNoFillComment

\Input{training data $D$, original ratio $w_{\ry, \rz}$, \revisionnew{$m$ samples of ($\ry_\text{deploy}$, $\rz_\text{deploy}$)}, confidence level of estimation $1-\delta$, thresholds $\gamma_\ry$ and $\gamma_\rz$,
in-processing algorithm\@ $f$}

\revisionnew{$\alpha, \beta$ = $\texttt{MLE}(\ry_\text{deploy}, \rz_\text{deploy}, 1-\delta)$ or given in advance}

$w'_{\ry, \rz}$ = $\texttt{SDPsolver}(w_{\ry, \rz}, \textcolor{black}{\alpha, \beta}, \gamma_\ry, \gamma_\rz)$

$d_j$ $\gets$ $w'_{\ry=y, \rz=z}/w_{\ry=y, \rz=z}, \forall j \in \mathbb{I}_{(y,z)}, \forall (y, z)\in \mathcal{\sY} \times \mathcal{\sZ}$

\textbf{d} = $\{d_i\}_{i=1,...,n}$

(Optional) \textbf{d} = $\texttt{MinDistChange}(D, w_{\ry, \rz}, w'_{\ry, \rz})$

Draw new data $D'$ from $D$ via weighted sampling w.r.t. \textbf{d}

$\theta$ $\gets$ initial model parameters

\textbf{for} \textit{each epoch} \textbf{do}: Update $\theta$ based on $f$ on $D'$

\Output{$\theta$}
\caption{Fair Training under Correlation Shifts}
\label{alg:overall}
\end{algorithm}


We present the overall process for fair training under correlation shifts in Algorithm~\ref{alg:overall}. The algorithm includes \textit{pre-processing} for reflecting the shifted correlation and \textit{in-processing} for fair training on top of the improved data. \revisionnew{During the pre-processing, we first obtain the shift range $[\alpha, \beta]$, which can be given in advance or inferred using $m$ samples and the target confidence level of estimation $1-\delta$ with a maximum likelihood estimator (MLE) -- see a detailed algorithm in Sec.~\ref{appendix:mle}}. We then find the new ($\ry$, $\rz$)-class ratio $w'_{\ry, \rz}$ based on the SDP relaxation of our optimization, which can be solved using convex optimization solvers (e.g., CVXPY). During the in-processing, we calculate the sample-wise weights to ensure that the sample weight sum in each ($\ry$, $\rz$)-class is $w_{\ry=y, \rz=z} \cdot n$, where $n$ is the total number of samples in the original training data. 
Within each ($\ry$, $\rz$)-class, all samples have the same weight. We then draw new data $D'$ from the original data via weighted sampling according to the sample-wise weights. Finally, we train a model using an in-processing fair algorithm $f$ on $D'$.



\textcolor{black}{In addition, using an optional step ($\texttt{MinDistChange}$), we can address the scenario where the pre-processing should minimally change the original training data, which is sometimes preferred in other applications~\citep{DBLP:journals/kais/KamiranC11}.
Details on this extension are in Sec.~\ref{appendix:min_dist_change}.}

\vspace{-0.3cm}
\section{Experiments}
\label{sec:experiments}
\vspace{-0.2cm}

We perform experiments to evaluate our framework. 
We use logistic regression in all experiments. We evaluate the models on separate test datasets and repeat all experiments with five different random seeds. We use CVXPY as a convex optimization solver -- see detailed settings in Sec.~\ref{appendix:ex_setting}.

\vspace{-0.35cm}
\paragraph{Fairness metrics} We focus on two prominent group fairness metrics: demographic parity (DP)~\citep{DBLP:conf/kdd/FeldmanFMSV15} and equalized odds (EO)~\citep{DBLP:conf/nips/HardtPNS16}.
We measure the fairness disparities (i.e., unfairness) among sensitive groups as follows: {\em DP disparity} = $\max_{z \in \mathcal{\sZ}}|\Pr(\hat{\ry}=1|\rz=z)-\Pr(\hat{\ry}=1)|$ and {\em EO disparity} = $\max_{z \in \mathcal{\sZ}, y \in \mathcal{\sY}}|\Pr(\hat{\ry}=y|\rz=z,\ry=y)-\Pr(\hat{\ry}=y|\ry=y)|$. 
When measuring the unfairness w.r.t. both DP and EO, we take the maximum of both disparities (i.e., $\max$({\em DP disp.}, {\em EO disp.})).
Note that lower disparity means better fairness.

\vspace{-0.35cm}
\paragraph{Datasets} 
We use a total of three datasets \textcolor{black}{for training}: one synthetic dataset and two real-world benchmark datasets. 
We generate the synthetic dataset using a method similar to \citet{DBLP:conf/aistats/ZafarVGG17}. 
The synthetic training dataset has 2,000 samples and consists of two non-sensitive attributes $(\rx_1, \rx_2)$, one sensitive attribute $\rz$, and one label attribute $\ry$ -- see details in Sec.~\ref{appendix:ex_setting}.
We also utilize two real datasets: \yuji{ProPublica COMPAS~\citep{Compas} consists of 5,278 samples, and its labels indicate recidivism; AdultCensus~\citep{DBLP:conf/kdd/Kohavi96} has 43,131 samples, and its labels indicate a person's annual income.}
We use gender as $\rz$.

\vspace{-0.05cm}
To construct test data representing the deployment distribution with shifted correlation \yuji{($c_\text{test}$)}, we use three methods: (1) re-sampling data within each ($\ry$, $\rz$)-class in the original test data, (2) modifying the $\rz$ values while fixing the $\rx$ and $\ry$ distributions in the original test data (see details in Sec.~\ref{appendix:ex_setting}), and (3) utilizing a newly-collected data, where we train on AdultCensus~\citep{DBLP:conf/kdd/Kohavi96}, but test on a recent version of this dataset called ACSIncome~\citep{ding2021retiring}.


\vspace{-0.35cm}
\paragraph{Baselines}
We compare our framework with three types of baselines: (1) \textit{vanilla (non-fair)} training using logistic regression, (2) \textit{in-processing-only} training, and (3) \textit{two-step} training that first runs an existing pre-processing algorithm and then an in-processing algorithm. 
For in-processing-only training, we use the following three approaches: Fairness Constraints (FC)~\citep{DBLP:conf/aistats/ZafarVGG17, DBLP:conf/www/ZafarVGG17}, which adds an unfairness penalty term to the loss function; Adversarial Debiasing (AD)~\citep{DBLP:conf/aies/ZhangLM18}, which adversarially trains a classifier with a fairness discriminator;
and FairBatch (FB)~\citep{roh2021fairbatch}, which adaptively adjusts batch ratios among groups to improve fairness.
\textcolor{black}{When we run in-processing approaches for multiple fairness metrics, we naturally extend each approach by combining the fairness constraints for different metrics -- see details in Sec.~\ref{appendix:ex_setting}.}
For two-step training, we use a pre-processing algorithm called Reweighing (RW)~\citep{DBLP:journals/kais/KamiranC11} to debias the data before running the above in-processings, where RW balances the data amounts across groups.
\textcolor{black}{We thus run the in-processing algorithms on the less-biased data by RW.}


\vspace{-0.35cm}
\paragraph{Hyperparameters} 
We consider \textit{four scenarios of knowing the range of the correlation constant $c$} in the test data. 
\textbf{(1)} In Secs.~\ref{sec:accfair} \&~\ref{sec:varying_corr}, we assume that \revisionnew{the exact $c_\text{test}$ value is given in advance (i.e., $\alpha {=} \beta {=} c_\text{test}$).
\textbf{(2)} In Secs.~\ref{sec:realworld} \&~\ref{appendix:newadult}, we use an \textit{estimated} $c$ range that is calculated using some of the new data.
\textbf{(3)} In Sec.~\ref{appendix:c_range}, we evaluate with  arbitrary $c$ ranges (i.e., $c_\text{test} \pm x\%$).}
\textbf{(4)} In Secs.~\ref{sec:misspecifying_corr} \&~\ref{appendix:unknown_corr_compas}, we assume that the correlation itself is \textit{incorrect}.
We set both $\gamma_\ry$ and $\gamma_\rz$ to 0.1, \textcolor{black}{which is a larger value than the actual marginal probability changes in the test data. We also test for the values of 0.2 and 0.3, and the overall trends remain the same.}
\yuji{The in-processings' hyperparameters are in Sec.~\ref{appendix:ex_setting}.}

\begin{table}[t]
\vspace{-0.2cm}
  \caption{\small Performances on the synthetic test dataset w.r.t. a single metric (DP) and multiple metrics (DP \& EO). The test datasets are constructed via re-sampling from the original distribution. 
  The correlation constant $c$ of the test data is 50\% of that of the training data. 
  We compare our framework with three types of baselines: (1) non-fair training: LR; (2) in-processing-only training: FC, AD, and FB; (3) two-step training: RW~\citep{DBLP:journals/kais/KamiranC11} + in-processings. 
  \yuji{In the last row, we also show the performances of an in-processing algorithm (FB) trained on the test distribution, which can be considered as the upper bounds.}
  }
  \vspace{-0.05cm}
  \label{tbl:synthetic}
  \centering
\scalebox{0.64}{
  \begin{tabular}{l@{\hspace{12pt}}c@{\hspace{7pt}}c@{\hspace{12pt}}c@{\hspace{7pt}}c}
    \toprule
      & \multicolumn{2}{c}{Single (DP)} & \multicolumn{2}{c}{Multiple (DP \& EO)}\\
    \cmidrule(r){1-5}
      Method & Acc. & Unfair. & Acc. & Unfair. \\
    \midrule
    LR & .865 $\pm$ .000 & .173 $\pm$ .000 & .865 $\pm$ .000 & .173 $\pm$ .000 \\
    \cmidrule(l){1-5}
    FC~\citep{DBLP:conf/aistats/ZafarVGG17, DBLP:conf/www/ZafarVGG17} & .778 $\pm$ .011 & .038 $\pm$ .013 & .853 $\pm$ .001 & .075 $\pm$ .002 \\
    RW+FC & .848 $\pm$ .004 & .079 $\pm$ .003 & .848 $\pm$ .004 & .082 $\pm$ .002 \\
    \textbf{Ours}+FC & .849 $\pm$ .002 & \textbf{.034 $\pm$ .004} & .851 $\pm$ .005 & \textbf{.060 $\pm$ .003} \\
    \cmidrule(l){1-5}
    AD~\citep{DBLP:conf/aies/ZhangLM18} & .762 $\pm$ .016 & .032 $\pm$ .011 & .821 $\pm$ .004 & .068 $\pm$ .003 \\
    RW+AD & .845 $\pm$ .002 & .087 $\pm$ .006 & .847 $\pm$ .003 & .084 $\pm$ .005 \\
    \textbf{Ours}+AD & .814 $\pm$ .011 & \textbf{.017 $\pm$ .006} & .842 $\pm$ .009 & \textbf{.054 $\pm$ .008} \\
    \cmidrule(l){1-5}
    FB~\citep{roh2021fairbatch} & .821 $\pm$ .000 & .048 $\pm$ .000 & .849 $\pm$ .001 & .091 $\pm$ .005 \\
    RW+FB & .859 $\pm$ .002 & .055 $\pm$ .003 & .855 $\pm$ .001 & .071 $\pm$ .008 \\
    \textbf{Ours}+FB & .836 $\pm$ .001 & \textbf{.003 $\pm$ .001} & .852 $\pm$ .004 & \textbf{.058 $\pm$ .001} \\
    \cmidrule{1-5}
    {\em FB on test dist. (upper bound)} & .838 $\pm$ .002 & .003 $\pm$ .002 & .859 $\pm$ .002 & .058 $\pm$ .004\\
    \bottomrule
  \end{tabular}
  }
  \vspace{-0.5cm}
\end{table}

\begin{table}[t]
\vspace{-0.25cm}
  \caption{\small Performances on the \textcolor{black}{COMPAS} test dataset. Other settings are identical to those in Table~\ref{tbl:synthetic}.}
  \label{tbl:compas}
  \centering
\scalebox{0.64}{
  \begin{tabular}{l@{\hspace{12pt}}c@{\hspace{7pt}}c@{\hspace{12pt}}c@{\hspace{7pt}}c}
    \toprule
      & \multicolumn{2}{c}{Single (DP)} & \multicolumn{2}{c}{Multiple (DP \& EO)}\\
    \cmidrule(r){1-5}
      Method & Acc. & Unfair. & Acc. & Unfair. \\
    \midrule
    LR & .660 $\pm$ .000 & .129 $\pm$ .000 & .660 $\pm$ .000 & .225 $\pm$ .000\\
    \cmidrule(l){1-5}
    FC~\citep{DBLP:conf/aistats/ZafarVGG17, DBLP:conf/www/ZafarVGG17} & .656 $\pm$ .004 & .050 $\pm$ .021 & .654 $\pm$ .007 & .137 $\pm$ .034\\
    RW+FC & .654 $\pm$ .004 & .068 $\pm$ .039 & .651 $\pm$ .010 & .124 $\pm$ .031\\
    \textbf{Ours}+FC & .657 $\pm$ .008 & \textbf{.037 $\pm$ .021} & .652 $\pm$ .020 & \textbf{.106 $\pm$ .038} \\
    \cmidrule(l){1-5}
    AD~\citep{DBLP:conf/aies/ZhangLM18} & .655 $\pm$ .003 & .054 $\pm$ .015 & .661 $\pm$ .005 & \textbf{.111 $\pm$ .035} \\
    RW+AD & .657 $\pm$ .005 & .092 $\pm$ .022 & .655 $\pm$ .007 & .145 $\pm$ .056 \\
    \textbf{Ours}+AD & .650 $\pm$ .003 & \textbf{.045 $\pm$ .008} & .664 $\pm$ .006 & .117 $\pm$ .039 \\
    \cmidrule(l){1-5}
    FB~\citep{roh2021fairbatch} & .647 $\pm$ .001 & .038 $\pm$ .013 & .650 $\pm$ .002 & .187 $\pm$ .019 \\
    RW+FB & .653 $\pm$ .003 & .094 $\pm$ .016 & .652 $\pm$ .003 & .197 $\pm$ .020\\
    \textbf{Ours}+FB & .648 $\pm$ .004 & \textbf{.027 $\pm$ .001} & .657 $\pm$ .004 & \textbf{.130 $\pm$ .014}\\
    \cmidrule{1-5}
    {\em FB on test dist. (upper bound)} & .659 $\pm$ .001 & .012 $\pm$ .008 & .664 $\pm$ .001 & .095 $\pm$ .014\\
    \bottomrule
  \end{tabular}
  }
\end{table}

\vspace{-0.15cm}
\subsection{Accuracy and Fairness}
\label{sec:accfair}
\vspace{-0.05cm}

We first compare the accuracy and fairness performances of our framework with baselines on the synthetic \textcolor{black}{and COMPAS datasets} in Tables~\ref{tbl:synthetic} and~\ref{tbl:compas} -- see \yuji{many more results in Appendix~\ref{appendix:experiments}, including the AdultCensus experiments (Sec.~\ref{appendix:adultcensus}), which show similar results}. The test data has lower correlation than the training data. 
LR shows vanilla training without any fairness technique. Other baselines are clustered based on three in-processing approaches: FC, AD, and FB. 
For each in-processing approach X, applying our pre-processing (denoted as Ours+X) generally shows better fairness while achieving comparable or even better accuracies, either when supporting only DP or both DP and EO. 
The in-processing-only baselines mostly show worse fairness and accuracy compared to applying our approach, \revision{as} the in-processing-only baselines are trained with different biases from test data distribution in mind. 
The baselines of applying RW before in-processing generally do not achieve high fairness compared to ours. 
The reason is that existing pre-processing approaches like RW simply mitigate data bias as much as possible, which is not always beneficial for the in-processing. 
In comparison, our approach takes a more principled approach by adjusting the bias according to the correlation shift with the purpose of improving the in-processing performance.
\yuji{As a result, ours enables the in-processing approaches to be closer to optimal performances (last row).}
In Sec.~\ref{appendix:align_well}, we show that the pre-processed data by our algorithm is \revision{indeed} \textit{more aligned} with the true test distribution than the original training data. 


We observe similar results when using the \textit{two other test settings} explained above. One is to modify the $\rz$ values while fixing the $\rx$ and $\ry$ distributions using the synthetic dataset (Sec.~\ref{appendix:test_construction}).
The other is to use two real-world income datasets collected in the 1990s and 2010s for training and testing, respectively, where they have different $(\ry, \rz)$-correlation values and also shifted $\rx$ distribution (Sec.~\ref{sec:realworld}).

\vspace{-0.05cm}
\revision{In addition, we enrich the empirical results by 1) presenting accuracy-fairness \textit{tradeoff curves} of the methods (Sec.~\ref{appendix:tradeoff}) and 2) showing the effects of the optional step in Algorithm~\ref{alg:overall} (Sec.~\ref{appendix:optional_step}), which finds possibly-different sample weights within each ($\ry, \rz$)-class to minimize the overall distribution change between the original and pre-processed data.}

\vspace{-0.3cm}
\subsection{Varying the Correlation of the Test Data}
\label{sec:varying_corr}
\vspace{-0.2cm}

We compare the algorithm performances when varying the correlation of the test data.
Figure~\ref{fig:varying_corr} shows the accuracy and fairness performances of FB and Ours+FB \revision{when the test data's correlation constant $c$ varies from 10\% to 70\% of the training data's $c$}. 
Interestingly, when the correlation of the test data differs significantly from the training data (i.e., close to 10\%), the in-processing-only baseline (FB) shows worse fairness. We suspect that the baseline is mitigating a different type of bias than that of the test data.
On the other hand, our pre-processing successfully enables the in-processing algorithm to achieve high accuracy and fairness (e.g., for a 10\% test corr., the unfairness decreases from 0.108 to 0.003).
In addition, we vary the correlation from 110\% to 150\% in Sec.~\ref{appendix:varying_corr_larger} where we show how our pre-processing enables in-processing to achieve high fairness.

\begin{figure}[h]
\vspace{-0.3cm}
\centering
\hspace{-0.2cm}
\includegraphics[width=0.97\columnwidth,trim=0cm 0.9cm 0cm 0cm]{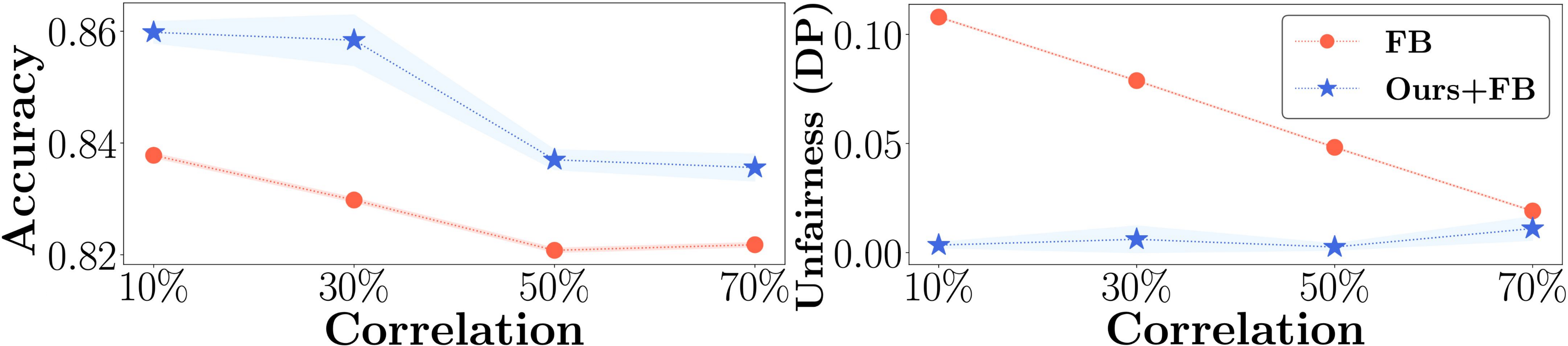}
\vspace{-0.3cm}
\caption{\small Performances of FB and Ours+FB on the synthetic data while varying the correlation of the test data.} 
\vspace{-0.3cm}
\label{fig:varying_corr}
\end{figure}


\vspace{-0.25cm}
\subsection{\resizebox{0.92\hsize}{!}{\revision{Estimating Correlations and Handling x Dist. Shifts}}}
\label{sec:realworld}
\vspace{-0.2cm}

\revision{We evaluate our method when 1) the shifted correlation is estimated from only a small amount of deployment data and 2) when the $\rx$ distribution changes. We use the two real-world income datasets collected in the 1990s (i.e., AdultCensus) and 2010s (i.e., ACSIncome) for training and testing, respectively. As the two datasets have a large time gap, they naturally have different ($\ry$, $\rz$)-correlations as well as input feature $\rx$ distributions (see Sec.~\ref{appendix:newadult} for an analysis).} \revisionnew{Here, we assume access to only 1,000 samples ($\simeq$ 0.07\%) of the 2010s data to measure the shifted correlation constant $c$ -- see the estimation performances when varying the number of samples in Sec.~\ref{appendix:newadult}. We compute the new $c$ value via a maximum likelihood estimator, as described in Sec.~\ref{sec:optimization}. As a result, the estimated $c$ is 0.140, where the true $c$ computed on the entire new data is 0.147. Here, we can find an error range of the estimated $c$ with a specific probability, as shown in Sec.~\ref{appendix:estimate_corr}. We note that we can use the error range as $[\alpha, \beta]$. When running our pre-processing with this estimation,}
\revision{our framework still outperforms the in-processing baselines (Table~\ref{tbl:new_adult}). There are \textbf{two takeaways}: 1) one can reasonably estimate the correlation shift} \revisionnew{via a practical amount of new data, and 2) our method performs well even when $\rx$ shifts.}



\vspace{-0.35cm}
\subsection{\yuji{Robustness Against Unknown Correlations}}
\label{sec:misspecifying_corr}
\vspace{-0.2cm}

\revision{Finally}, we evaluate our approach when the exact shifted correlation is unknown with two scenarios: 1) misspecifying the correlation in the algorithm and 2) giving a range of correlation shifts to the algorithm.
\textbf{[Scenario 1]} We first consider when the shifted correlation is \textit{incorrectly specified}. We set $\alpha {=} \beta {=} c_\text{specified}$, where $c_\text{specified} {\neq} c_\text{test}$. Figure~\ref{fig:wrong_corr} shows the performances of FB and Ours+FB when the true correlation of the test data is 60\% of the training data's correlation. Ours improves the in-processing-only baseline’s accuracy for the entire range of considered correlations and improves fairness when the specified correlation is higher than 30\%. Hence, our approach is still beneficial when the estimation error is within 10\%.
\textbf{[Scenario 2]} In Sec.~\ref{appendix:c_range}, we also run our approach with \textit{a range of correlation shifts}.
\revisionnew{Here we set the arbitrary $[\alpha, \beta]$ range to be $[c_\text{test} {-} x\%, c_\text{test} {+} x\%]$ with various $x$ values to see the robustness of our method.}
As a result, there are \textbf{two takeaways}: 1) our framework successfully boosts the in-processing-only baseline performances when the $[\alpha, \beta]$ range is reasonable, and 2) even if we do not have any information about the shift, our framework performs at least as well as the in-processing-only baselines. \revision{We observe \textit{similar results on real-world data} in Sec.~\ref{appendix:unknown_corr_compas}.}

\begin{figure}[h]
\vspace{-0.25cm}
\centering
\hspace{-0.2cm}
\includegraphics[width=0.97\columnwidth,trim=0cm 0.9cm 0cm 0cm]{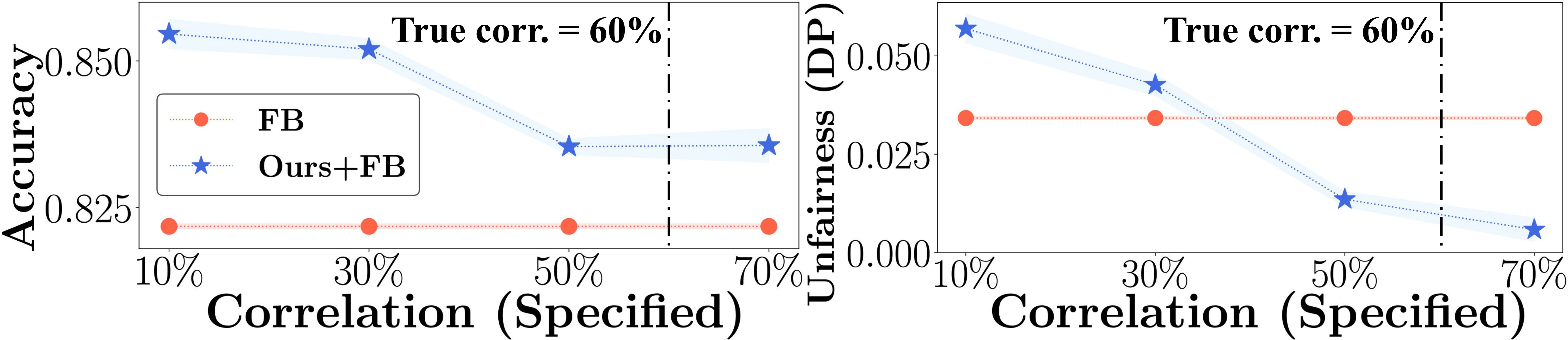}
\vspace{-0.35cm}
\caption{\small Performances of FB and Ours+FB on the synthetic data while varying the \textit{specified} correlation in our algorithm to have 10\% to 70\% where the true correlation of the test data is 60\%.}
\vspace{-0.3cm}
\label{fig:wrong_corr}
\end{figure}

\vspace{-0.25cm}
\section{Related Work}
\label{sec:relatedwork}
\vspace{-0.2cm}

As model fairness becomes essential for Trustworthy AI, various fairness techniques have been proposed to improve group fairness that do not discriminate specific demographics~\citep{barocas-hardt-narayanan}. 
The fairness techniques can be categorized into three prominent approaches: pre-, in-, and post-processings -- see their representative works in Sec.~\ref{appendix:relatedwork_traditional_fairness}. Among the three categories, in-processing approaches are widely used for fair training due to their high fairness and accuracy performances, but most of them assume that the training and deployment distributions are the same.
\vspace{-0.05cm}

\setlength{\columnsep}{4pt}
\begin{wraptable}{r}{5.0cm}
\vspace{-0.85cm}
  \caption{\small Different distribution shifts.}
  \label{tbl:related_work}
  \centering
\scalebox{0.74}{
  \begin{tabular}{c@{\hspace{5pt}}c@{\hspace{5pt}}c}
    \toprule
    Category & \multicolumn{2}{c}{Type of Shifts} \\
    \midrule
    \multirow{3}{*}{\breakcell{General\\ distribution\\ shifts}} & covariate shift & $\Pr(\rx)$\\
    & label shift & $\Pr(\ry)$\\
    & concept shift & $\Pr(\ry|\rx)$\\
    \cmidrule{1-3}
    \multirow{3}{*}{\breakcell{Fairness\\ specific\\ shifts}} & demographic shift & $\Pr(\rz)$ \\
    & subpopulation shift & $\Pr(\ry,\rz)$ \\
    & \textbf{corr. shift (ours)} & $\Pr(\rz|\ry)$ \\
    \bottomrule
  \end{tabular}
  }
  \vspace{-0.65cm}
\end{wraptable}
Recently, there is an emerging focus on \revision{fairness} under data distri-bution shifts. In Table \ref{tbl:related_work}, we summarize various types of shifts into the two categories: 
\vspace{-0.1cm}

1. {\em General distribution shifts}~\citep{singh2021fairness, Rezaei_Liu_Memarrast_Ziebart_2021, chen2022fairness, mishler2022fair} focus on shifts involving the input feature ($\rx$) and label ($\ry$), which are widely studied in the traditional machine learning literature. Here the bias changes between label ($\ry$) and group ($\rz$) are not explicitly considered.

\vspace{-0.1cm}
2. {\em Fairness-specific shifts}~\citep{maity2021does, an2022transferring, giguere2022fairness, 51492} handle group ($\rz$) changes, as $\rz$ is especially correlated with fair training. Our work also falls into this category. A recent study~\citep{maity2021does} theoretically analyzes the behavior of fair training under a change in bias called subpopulation shifts, where a specific group has fewer positively-labeled examples during training time compared to deployment time. Another study~\citep{giguere2022fairness} designs a new test method to serve a fair model under another distribution change called demographic shifts, where the subgroup distribution may change \revisionnew{-- see an \textit{empirical comparison} with our work in Sec.~\ref{appendix:shifty}}. A recent work~\citep{an2022transferring} proposes a self-training-based transfer algorithm that requires specific model architecture (e.g., adversary network) to support data changes.
In comparison, our contribution lies in 1) introducing the notion of correlation shifts, which is important for explaining with theoretical evidence the connection between the data bias changes and behaviors of fair training, 2) analyzing the fundamental limits of in-processing approaches \revision{under} correlation shifts, and 3) proposing a pre-processing step based on the theoretical analysis for assisting the existing fairness approaches. In addition, our framework is general and can \textit{support any model architecture and training procedure}.
\revision{We leave more detailed comparisons in Sec.~\ref{appendix:relatedwork_distribution_shifts}.}

Another line of research is \revision{supporting robustness in fair training, including handling noisy groups~\citep{Celis2021FairCW, NEURIPS2020_37d097ca} or poisoning attacks~\citep{mehrabi2021exacerbating, DBLP:conf/pkdd/SolansB020}}. Although this direction is not our immediate focus, we do perform preliminary experiments in Secs.~\ref{appendix:noisy} \revision{and~\ref{appendix:poisoning}} to show some potential to support noisy group attributes \revision{or poisoning attack scenarios}. 

\revision{In addition to group fairness, we discuss other noteworthy fairness definitions including causality-based fairness~\citep{10.5555/3294771.3294834} in Secs.~\ref{appendix:relatedwork_traditional_fairness} and~\ref{appendix:causality}.}

\vspace{-0.3cm}
\section{Conclusion}
\label{sec:conclusion}
\vspace{-0.15cm}

We addressed the problem of model fairness in the presence of bias changes in the data.
We first introduced the new notion of ($\ry, \rz$)-correlation for capturing bias and analyzed the \revision{performance limits} of in-processing approaches in the presence of correlation shifts.
We then proposed a decoupling framework where pre-processing is used to adjust the correlation, and in-processing is used for unfairness mitigation.
Experiments showed how our pre-processing enables existing in-processing approaches to achieve high fairness and accuracy under correlation shifts and outperform baselines.

\bibliography{main}
\bibliographystyle{icml2023}

\newpage
\appendix
\onecolumn

\section{Appendix -- Theory}

\subsection{\revisionnew{Toy Example}}
\label{appendix:toyexample}

\revisionnew{In fair training, the data bias reflects the relationship between a label $\ry$ and group attribute $\rz$. For example, if all positive labels are in the same group, the data can be considered highly biased. Conversely, if the labels are randomly assigned to \textcolor{black}{all} groups, the data can be considered unbiased.}

\revisionnew{A bias change in the deployment data may have an adverse affect on a trained model's performance.
Figure~\ref{fig:toyexample} shows a toy example that illustrates how a fair classifier's performance is affected by a bias change during deployment. 
Here, the bias can be expressed via correlation, and we discuss their relationship in Sec.~\ref{sec:correlationshift}.
The training distribution is biased, where Group 2 has more positive labels than Group 1. On the other hand, in the deployment distribution, the bias decreases as the labels are equally distributed for each group. 
Suppose we train a fair classifier on the training data as shown on the left side where the DP fairness is perfect (i.e., $\Pr(\hat{\ry}{=}1)$ are the same for the groups), but the accuracy is not as a result. On the deployment data, the DP worsens while the accuracy still remains imperfect. The underlying problem is that the classifier was trained with a different bias in mind.}

\begin{figure}[h]
\centering
\includegraphics[width=0.55\columnwidth,trim=0cm 0.3cm 0cm 0cm]{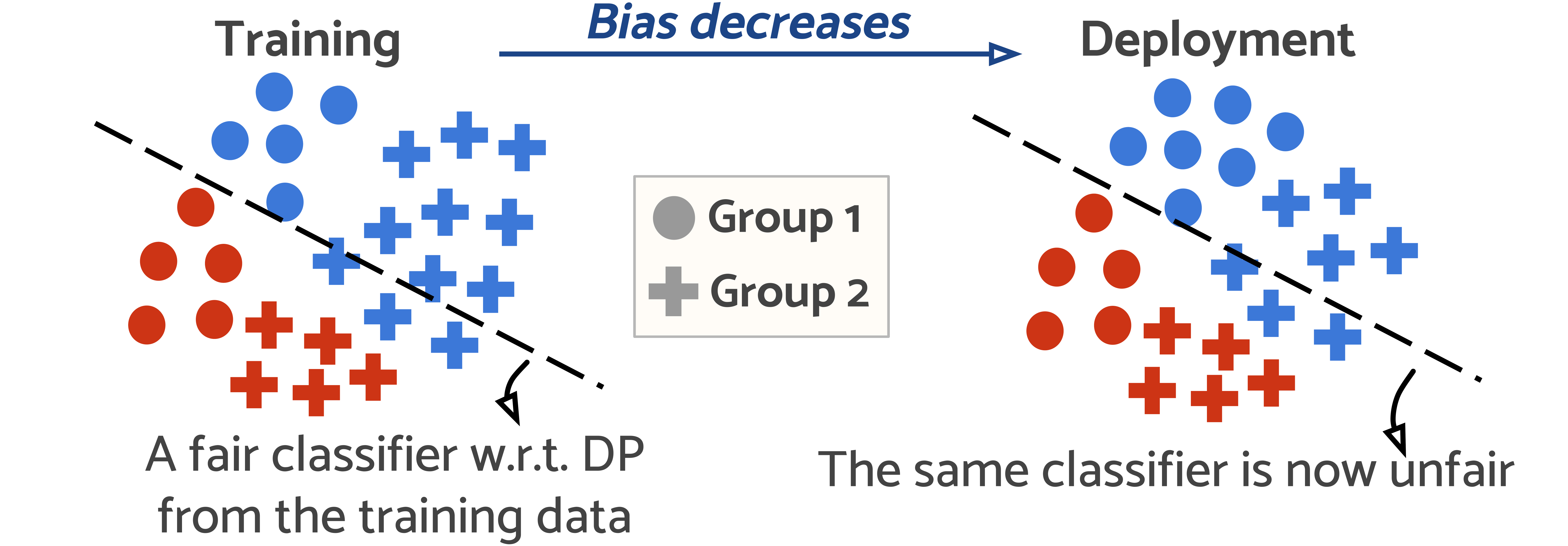}
\vspace{-0.2cm}
\caption{\revisionnew{A toy example for illustrating the impact of bias changes. The blue and red colors indicate positive and negative examples, respectively. On the left, there is a fair classifier that achieves DP on the training distribution. On the right, the deployment data has a lower bias, making the trained classifier both unfair and inaccurate.}}
\label{fig:toyexample}
\end{figure}

\subsection{Proof for Lemma~\ref{lemma:proportional}}
\label{appendix:correlationproperty}

Continuing from Sec.~\ref{sec:fundamentallimits}, we provide a proof for Lemma~\ref{lemma:proportional}.
\begin{proof}

We denote the Pearson's correlation coefficient between $\ry$ and $\rz$ as $\rho_{\ry\rz}$.
By definition, $\rho_{\ry\rz} = \frac{Cov(\ry, \rz)}{\sigma(\ry)\sigma(\rz)}$.
For binary $\ry$ and $\rz$, we can rewrite $\rho_{\ry\rz}$ as follows~\citep{cohen1975applied}:
\begin{align*}
    \rho_{\ry\rz} = \frac{Cov(\ry, \rz)}{\sigma(\ry)\sigma(\rz)}
    = \frac{\Pr(\ry=1, \rz=1)\Pr(\ry=0, \rz=0) - \Pr(\ry=1, \rz=0)\Pr(\ry=0, \rz=1)}{\sqrt{\Pr(\ry=1)\Pr(\ry=0)\Pr(\rz=1)\Pr(\rz=0)}}
\end{align*}


If the marginal probabilities of $\ry$ and $\rz$ (i.e., $\Pr(\ry=y)$ and $\Pr(\rz=z)$) remain the same, 
\begin{align*}
    \rho_{\ry\rz} &\propto \Pr(\ry=1, \rz=1)\Pr(\ry=0, \rz=0) - \Pr(\ry=1, \rz=0)\Pr(\ry=0, \rz=1)\\
    &\propto  \frac{\Pr(\ry=1, \rz=1)}{\Pr(\rz=1)}\frac{\Pr(\ry=0, \rz=0)}{\Pr(\rz=0)} - \frac{\Pr(\ry=1, \rz=0)}{\Pr(\rz=0)}\frac{\Pr(\ry=0, \rz=1)}{\Pr(\rz=1)}\\
    &= \Pr(\ry=1|\rz=1)\Pr(\ry=0|\rz=0) - \Pr(\ry=1|\rz=0)\Pr(\ry=0|\rz=1)\\
    &= \Pr(\ry=1|\rz=1)(1-\Pr(\ry=1|\rz=0)) - \Pr(\ry=1|\rz=0)(1-\Pr(\ry=1|\rz=1))\\
    &= \Pr(\ry=1|\rz=1) - \Pr(\ry=1|\rz=0) - \Pr(\ry=1|\rz=1)\Pr(\ry=1|\rz=0)+\Pr(\ry=1|\rz=0)\Pr(\ry=1|\rz=1)\\
    &= \Pr(\ry=1|\rz=1) - \Pr(\ry=1|\rz=0).
\end{align*}

Therefore, if the marginal probabilities of $\ry$ and $\rz$ remain the same, the ($\ry, \rz$)-correlation $\rho_{\ry\rz}$ is proportional to $\Pr(\ry=1|\rz=1) - \Pr(\ry=1|\rz=0)$, which is the difference between the conditional probabilities of $\ry$ given different $\rz$ values.

\end{proof}

\subsection{The Accuracy-Fairness Tradeoff When Improving Fairness w.r.t. a Single Metric}
\label{appendix:singlemetric}

Continuing from Sec.~\ref{sec:fundamentallimits}, we show that the ($\ry, \rz$)-correlation determines the accuracy-fairness tradeoff under certain conditions by applying Lemma~\ref{lemma:proportional} to Proposition~\ref{prop:menon}. 

The previous work~\citep{pmlr-v81-menon18a} shows that a higher alignment between $\ry$ and $\rz$ leads to a worse accuracy-fairness tradeoff w.r.t. demographic parity -- more details are described in Proposition 8 in~\citet{pmlr-v81-menon18a}. 
According to the previous work, alignment is defined as how many examples in each group have a specific label. We can thus measure the alignment as $\Pr(\ry=1|\rz=1) + \Pr(\ry=0|\rz=0)$. 
Then, we can rewrite the equation as follows:
\begin{align*}
    \text{alignment} &= \Pr(\ry=1|\rz=1) + \Pr(\ry=0|\rz=0)\\
    &= \Pr(\ry=1|\rz=1) + (1 - \Pr(\ry=1|\rz=0))\\
    &= 1 + \Pr(\ry=1|\rz=1) - \Pr(\ry=1|\rz=0).
\end{align*}

By applying Lemma~\ref{lemma:proportional} to the above result, we observe that the alignment between $\ry$ and $\rz$ is proportional to (($\ry, \rz$)-correlation $+ 1$) when the marginal distributions of $\ry$ and $\rz$ remain the same.
Therefore, when a model is trained w.r.t. demographic parity and the marginal probabilities of $\ry$ and $\rz$ do not change, a higher ($\ry, \rz$)-correlation results in a worse accuracy-fairness tradeoff.





\subsection{$\varepsilon$-DP \& $\varepsilon$-EO}
\label{appendix:dp_eo}

Continuing from Sec.~\ref{sec:fundamentallimits}, we provide a proof for Proposition~\ref{prop:dpeo}.

\begin{proof}

A model achieves both $\varepsilon$-DP and $\varepsilon$-EO when the two inequalities $|\Pr(\hat{\ry}=y|\rz=z) - \Pr(\hat{\ry}=y|\rz=z')| \leq \varepsilon$ (i.e., $\varepsilon$-DP) and $|\Pr(\hat{\ry}=y|\ry=y, \rz=z) - \Pr(\hat{\ry}=y|\ry=y, \rz=z')| \leq \varepsilon$ (i.e., $\varepsilon$-EO) are satisfied, where $y \neq y', y, y' \in \{0,1\}$ and $z \neq z', z, z' \in \{0,1\}$. 

By combining the law of total probability and $\varepsilon$-DP, we can get the following inequality:
\begin{align}
    -\varepsilon \leq \Pr(\hat{\ry}=y|\rz=z, \ry=y') \Pr(\ry=y'|\rz=z) &+ \Pr(\hat{\ry}=y|\rz=z, \ry=y) \Pr(\ry=y|\rz=z) \nonumber \\ 
    - \Pr(\hat{\ry}=y|\rz=z', \ry=y') \Pr(\ry=y'|\rz=z') &- \Pr(\hat{\ry}=y|\rz=z', \ry=y) \Pr(\ry=y|\rz=z') \leq \varepsilon . \label{eq:dp_modif}
\end{align}

Also, $\varepsilon$-EO can be rewritten as follows:
\begin{align}
    -\varepsilon + \Pr(\hat{\ry}=y|\rz=z',\ry=y) \leq \Pr(\hat{\ry}=y|\rz=z,\ry=y)  \leq \varepsilon + \Pr(\hat{\ry}=y|\rz=z',\ry=y). \label{eq:eo_modif}
\end{align}

By substituting Eq.~\ref{eq:eo_modif} to Eq.~\ref{eq:dp_modif}, we can get the following inequality:
\begin{align}
    &(-\varepsilon + \Pr(\hat{\ry}=y|\rz=z',\ry=y'))\Pr(\ry=y'|\rz=z) + (-\varepsilon + \Pr(\hat{\ry}=y|\rz=z',\ry=y)) \Pr(\ry=y|\rz=z) \nonumber \\  
    &- \Pr(\hat{\ry}=y|\rz=z', \ry=y') \Pr(\ry=y'|\rz=z') - \Pr(\hat{\ry}=y|\rz=z', \ry=y) \Pr(\ry=y|\rz=z') \nonumber \\  
    &\leq \Pr(\hat{\ry}=y|\rz=z, \ry=y') \Pr(\ry=y'|\rz=z) + \Pr(\hat{\ry}=y|\rz=z, \ry=y) \Pr(\ry=y|\rz=z) \label{eq:dp_eo_sub} \\
    &- \Pr(\hat{\ry}=y|\rz=z', \ry=y') \Pr(\ry=y'|\rz=z') - \Pr(\hat{\ry}=y|\rz=z', \ry=y) \Pr(\ry=y|\rz=z')  \nonumber \\ 
    &\leq (\varepsilon + \Pr(\hat{\ry}=y|\rz=z',\ry=y'))\Pr(\ry=y'|\rz=z) + (\varepsilon + \Pr(\hat{\ry}=y|\rz=z',\ry=y)) \Pr(\ry=y|\rz=z) \nonumber \\ 
    &- \Pr(\hat{\ry}=y|\rz=z', \ry=y') \Pr(\ry=y'|\rz=z') - \Pr(\hat{\ry}=y|\rz=z', \ry=y) \Pr(\ry=y|\rz=z'). \nonumber 
\end{align}

By subtracting Eq.~\ref{eq:eo_modif} from Eq.~\ref{eq:dp_eo_sub},
\begin{align*}
    &-\varepsilon + (-\varepsilon + \Pr(\hat{\ry}=y|\rz=z',\ry=y'))\Pr(\ry=y'|\rz=z) + (-\varepsilon + \Pr(\hat{\ry}=y|\rz=z',\ry=y)) \Pr(\ry=y|\rz=z) \\
    &- \Pr(\hat{\ry}=y|\rz=z', \ry=y') \Pr(\ry=y'|\rz=z') - \Pr(\hat{\ry}=y|\rz=z', \ry=y) \Pr(\ry=y|\rz=z') \\
    &\leq 0  \\
    &\leq \varepsilon + (\varepsilon + \Pr(\hat{\ry}=y|\rz=z',\ry=y'))\Pr(\ry=y'|\rz=z) + (\varepsilon + \Pr(\hat{\ry}=y|\rz=z',\ry=y)) \Pr(\ry=y|\rz=z) \\
    &- \Pr(\hat{\ry}=y|\rz=z', \ry=y') \Pr(\ry=y'|\rz=z') - \Pr(\hat{\ry}=y|\rz=z', \ry=y) \Pr(\ry=y|\rz=z').
\end{align*}

By rearranging the terms, we get the followings:
\begin{align*}
    -2\varepsilon \leq (\Pr(\hat{\ry}=y|\rz=z',\ry=y')-\Pr(\hat{\ry}=y|\rz=z',\ry=y))(\Pr(\ry=y'|\rz=z)-\Pr(\ry=y'|\rz=z')) \leq 2\varepsilon \\
    \Rightarrow -2\varepsilon \leq (\Pr(\hat{\ry}=y|\rz=z',\ry=y')-\Pr(\hat{\ry}=y|\rz=z',\ry=y))(1-\Pr(\ry=y|\rz=z)-1+\Pr(\ry=y|\rz=z')) \leq 2\varepsilon.
\end{align*}


Thus, we get the following inequality, which is in Proposition~\ref{prop:dpeo}:
\begin{align*}
     |\Pr(\ry=y|\rz=z)-\Pr(\ry=y|\rz=z')|~|\Pr(\hat{\ry}=y|\ry=y,\rz=z)-\Pr(\hat{\ry}=y|\ry=y',\rz=z)| \leq 2\varepsilon.
\end{align*}

\end{proof}





\subsection{$\varepsilon$-PP \& $\varepsilon$-DP}
\label{appendix:pp_dp}

Continuing from Sec.~\ref{sec:fundamentallimits}, we consider fair training w.r.t. both predictive parity (PP) and demographic parity (DP). 
We can define $\varepsilon$-PP \& $\varepsilon$-DP as follows:
\begin{align*}
    -\varepsilon \leq \Pr(\ry=y|\hat{\ry}=y, \rz=0) &- \Pr(\ry=y|\hat{\ry}=y, \rz=1) \leq \varepsilon ~~\text{(i.e., $\varepsilon$-PP)}\\
    -\varepsilon \leq \Pr(\hat{\ry}=y|\rz=0) &- \Pr(\hat{\ry}=y| \rz=1) \leq \varepsilon ~~\text{(i.e., $\varepsilon$-DP)}
\end{align*}

Based on these definitions, we give a proposition for $\varepsilon$-PP and $\varepsilon$-DP:
\begin{proposition}[$\varepsilon$-PP \& $\varepsilon$-DP]\label{prop:ppdp}
Let a model achieve both $\varepsilon$-PP and $\varepsilon$-DP. Then, the following inequality holds:
\begin{align*}
    \frac{|\Pr(\ry=y,\hat{\ry}=y|\rz=z) - \Pr(\ry=y,\hat{\ry}=y|\rz=z')|}{2 \Pr(\ry=y|\hat{\ry}=y, \rz=z)+\Pr(\hat{\ry}=y|\rz=z)+\Pr(\ry=y|\hat{\ry}=y, \rz=z')} \leq \varepsilon, ~~
     z\neq z',~z,z' \in \mathcal{\sZ},~y \in \mathcal{\sY}.    
\end{align*}
\end{proposition}

\begin{proof}

From $\varepsilon$-DP, we get 
\begin{align}
    -\varepsilon +\Pr(\hat{\ry}=y| \rz=1) \leq \Pr(\hat{\ry}=y|\rz=0) \leq \varepsilon + \Pr(\hat{\ry}=y| \rz=1). \label{eq:ppdp_1}    
\end{align}

From $\varepsilon$-PP, we get 
{\small  
\begin{align}
-\varepsilon
    \leq \Pr(\hat{\ry}{=}y|\rz{=}0) \frac{\Pr(\rz{=}0)}{\Pr(\hat{\ry}{=}y,\rz{=}0)} \frac{\Pr(\ry{=}y,\hat{\ry}{=}y, \rz{=}0)}{\Pr(\hat{\ry}{=}y,\rz{=}0)} - \Pr(\hat{\ry}{=}y|\rz{=}1) \frac{\Pr(\rz{=}1)}{\Pr(\hat{\ry}{=}y,\rz{=}1)} \frac{\Pr(\ry{=}y,\hat{\ry}{=}y, \rz{=}1)}{\Pr(\hat{\ry}{=}y,\rz{=}1)}
    \leq \varepsilon. \label{eq:ppdp_2} 
\end{align}
}%

By substituting Eq.~\ref{eq:ppdp_1} to Eq.~\ref{eq:ppdp_2}, we can get the following inequality:
{\small 
\begin{align}
    &\{-\varepsilon +\Pr(\hat{\ry}=y|\rz=1)\} \frac{\Pr(\rz=0)}{\Pr(\hat{\ry}=y,\rz=0)} \frac{\Pr(\ry=y,\hat{\ry}=y, \rz=0)}{\Pr(\hat{\ry}=y,\rz=0)} - \Pr(\hat{\ry}=y|\rz=1) \frac{\Pr(\rz=1)}{\Pr(\hat{\ry}=y,\rz=1)} \frac{\Pr(\ry=y,\hat{\ry}=y, \rz=1)}{\Pr(\hat{\ry}=y,\rz=1)} \nonumber \\
    &\leq \Pr(\hat{\ry}=y|\rz=0) \frac{\Pr(\rz=0)}{\Pr(\hat{\ry}=y,\rz=0)} \frac{\Pr(\ry=y,\hat{\ry}=y, \rz=0)}{\Pr(\hat{\ry}=y,\rz=0)} - \Pr(\hat{\ry}=y|\rz=1) \frac{\Pr(\rz=1)}{\Pr(\hat{\ry}=y,\rz=1)} \frac{\Pr(\ry=y,\hat{\ry}=y, \rz=1)}{\Pr(\hat{\ry}=y,\rz=1)} \label{eq:ppdp_3}\\
    &\leq \{\varepsilon +\Pr(\hat{\ry}=y|\rz=1)\} \frac{\Pr(\rz=0)}{\Pr(\hat{\ry}=y,\rz=0)} \frac{\Pr(\ry=y,\hat{\ry}=y, \rz=0)}{\Pr(\hat{\ry}=y,\rz=0)} - \Pr(\hat{\ry}=y|\rz=1) \frac{\Pr(\rz=1)}{\Pr(\hat{\ry}=y,\rz=1)} \frac{\Pr(\ry=y,\hat{\ry}=y, \rz=1)}{\Pr(\hat{\ry}=y,\rz=1)}.\nonumber
\end{align}
}%

By subtracting Eq.~\ref{eq:ppdp_2} from Eq.~\ref{eq:ppdp_3},
{\small 
\begin{align*} 
    &-\varepsilon + \{-\varepsilon +\Pr(\hat{\ry}{=}y|\rz{=}1)\} \frac{\Pr(\rz=0)}{\Pr(\hat{\ry}=y,\rz=0)} \frac{\Pr(\ry=y,\hat{\ry}=y, \rz=0)}{\Pr(\hat{\ry}=y,\rz=0)} - \Pr(\hat{\ry}{=}y|\rz{=}1) \frac{\Pr(\rz=1)}{\Pr(\hat{\ry}=y,\rz=1)} \frac{\Pr(\ry=y,\hat{\ry}=y, \rz=1)}{\Pr(\hat{\ry}=y,\rz=1)} \\
    &\leq 0 \\
    &\leq \varepsilon + \{\varepsilon +\Pr(\hat{\ry}{=}y|\rz{=}1)\} \frac{\Pr(\rz=0)}{\Pr(\hat{\ry}=y,\rz=0)} \frac{\Pr(\ry=y,\hat{\ry}=y, \rz=0)}{\Pr(\hat{\ry}=y,\rz=0)} - \Pr(\hat{\ry}{=}y|\rz{=}1) \frac{\Pr(\rz=1)}{\Pr(\hat{\ry}=y,\rz=1)} \frac{\Pr(\ry=y,\hat{\ry}=y, \rz=1)}{\Pr(\hat{\ry}=y,\rz=1)}.
\end{align*}
}%

By rearranging the terms, we get the following inequality:
\begin{align*}
    &-\varepsilon \cdot \frac{\Pr(\rz=0)\Pr(\ry=y,\hat{\ry}=y, \rz=0)}{(\Pr(\hat{\ry}=y,\rz=0))^2} -\varepsilon \\
    &\leq \Pr(\hat{\ry}=y|\rz=1) \{ \frac{\Pr(\rz=1)\Pr(\ry=y,\hat{\ry}=y, \rz=1)}{(\Pr(\hat{\ry}=y,\rz=1))^2} - \frac{\Pr(\rz=0)\Pr(\ry=y,\hat{\ry}=y, \rz=0)}{(\Pr(\hat{\ry}=y,\rz=0))^2}\} \\
    &\leq \varepsilon \cdot \frac{\Pr(\rz=0)\Pr(\ry=y,\hat{\ry}=y, \rz=0)}{(\Pr(\hat{\ry}=y,\rz=0))^2} +\varepsilon
\end{align*}

which is rewritten as follows:
\begin{align*}
    &-\varepsilon \cdot \frac{\Pr(\ry=y|\hat{\ry}=y, \rz=0)+\Pr(\hat{\ry}=y|\rz=0)}{\Pr(\hat{\ry}=y|\rz=0)} \\
    &\leq \Pr(\hat{\ry}=y|\rz=1) \{ \frac{\Pr(\ry=y|\hat{\ry}=y, \rz=1)}{\Pr(\hat{\ry}=y|\rz=1)} - \frac{\Pr(\ry=y|\hat{\ry}=y, \rz=0)}{\Pr(\hat{\ry}=y|\rz=0)}\} \\
    &\leq \varepsilon \cdot \frac{\Pr(\ry=y|\hat{\ry}=y, \rz=0)+\Pr(\hat{\ry}=y|\rz=0)}{\Pr(\hat{\ry}=y|\rz=0)}.
\end{align*}

By multiplying $\frac{\Pr(\hat{\ry}=y|\rz=0)}{\Pr(\ry=y|\hat{\ry}=y, \rz=0)+\Pr(\hat{\ry}=y|\rz=0)}$ for all terms in the above inequality, we can get
\begin{align}
    -\varepsilon
    \leq \frac{\Pr(\hat{\ry}=y|\rz=0)\Pr(\ry=y|\hat{\ry}=y, \rz=1) - \Pr(\hat{\ry}=y|\rz=1)\Pr(\ry=y|\hat{\ry}=y, \rz=0)}{\Pr(\ry=y|\hat{\ry}=y, \rz=0)+\Pr(\hat{\ry}=y|\rz=0)}
    \leq \varepsilon. \label{eq:ppdp_4}
\end{align}

Let $A = \Pr(\ry=y|\hat{\ry}=y, \rz=0)+\Pr(\hat{\ry}=y|\rz=0)$. Since $-\varepsilon \leq \Pr(\hat{\ry}=y|\rz=0) - \Pr(\hat{\ry}=y| \rz=1) \leq \varepsilon$ (i.e., $\varepsilon$-DP), we can make another inequality from Eq.~\ref{eq:ppdp_4}:
\begin{align}
    &\frac{(-\varepsilon+\Pr(\hat{\ry}=y|\rz=1))\Pr(\ry=y|\hat{\ry}=y, \rz=1) - (\varepsilon+\Pr(\hat{\ry}=y|\rz=0))\Pr(\ry=y|\hat{\ry}=y, \rz=0)}{A} \nonumber \\
    &\leq \frac{\Pr(\hat{\ry}=y|\rz=0)\Pr(\ry=y|\hat{\ry}=y, \rz=1) - \Pr(\hat{\ry}=y|\rz=1)\Pr(\ry=y|\hat{\ry}=y, \rz=0)}{A} \label{eq:ppdp_5}\\
    &\leq \frac{(\varepsilon+\Pr(\hat{\ry}=y|\rz=1))\Pr(\ry=y|\hat{\ry}=y, \rz=1) - (-\varepsilon+\Pr(\hat{\ry}=y|\rz=0))\Pr(\ry=y|\hat{\ry}=y, \rz=0)}{A}. \nonumber
\end{align}

By subtracting Eq.~\ref{eq:ppdp_4} and Eq.~\ref{eq:ppdp_5},
\begin{align*}
    &-\varepsilon+\frac{(-\varepsilon+\Pr(\hat{\ry}=y|\rz=1))\Pr(\ry=y|\hat{\ry}=y, \rz=1) - (\varepsilon+\Pr(\hat{\ry}=y|\rz=0))\Pr(\ry=y|\hat{\ry}=y, \rz=0)}{A}\\
    &\leq 0\\
    &\leq \varepsilon+\frac{(\varepsilon+\Pr(\hat{\ry}=y|\rz=1))\Pr(\ry=y|\hat{\ry}=y, \rz=1) - (-\varepsilon+\Pr(\hat{\ry}=y|\rz=0))\Pr(\ry=y|\hat{\ry}=y, \rz=0)}{A}.
\end{align*}

Since $\Pr(\hat{\ry}=y|\rz=z)\Pr(\ry=y|\hat{\ry}=y, \rz=z) = \Pr(\ry=y,\hat{\ry}=y|\rz=z)$, we can rewritten the above inequality as follows:
\begin{align*}
    &-\varepsilon + \frac{-\varepsilon\{\Pr(\ry=y|\hat{\ry}=y, \rz=1) + \Pr(\ry=y|\hat{\ry}=y, \rz=0)\} + \Pr(\ry=y,\hat{\ry}=y|\rz=1) - \Pr(\ry=y,\hat{\ry}=y|\rz=0)}{A}\\
    &\leq 0 \\
    &\leq \varepsilon + \frac{\varepsilon\{\Pr(\ry=y|\hat{\ry}=y, \rz=1) + \Pr(\ry=y|\hat{\ry}=y, \rz=0)\} + \Pr(\ry=y,\hat{\ry}=y|\rz=1) - \Pr(\ry=y,\hat{\ry}=y|\rz=0)}{A}.
\end{align*}

Now, let $B {=} \Pr(\ry=y|\hat{\ry}=y, \rz=1) + \Pr(\ry=y|\hat{\ry}=y, \rz=0)$ and $C {=} \Pr(\ry=y,\hat{\ry}=y|\rz=1) - \Pr(\ry=y,\hat{\ry}=y|\rz=0)$. Then,
\begin{align*}
    -\varepsilon + \frac{1}{A}(-\varepsilon \cdot B + C) \leq 0 \leq \varepsilon + \frac{1}{A}(\varepsilon \cdot B + C).
\end{align*}

By arranging the terms,
\begin{align*}
    -\varepsilon \cdot (1 &+ \frac{B}{A}) + \frac{C}{A} \leq 0 \leq \varepsilon \cdot (1 + \frac{B}{A}) + \frac{C}{A} \\
    &\Longrightarrow -\varepsilon \leq \frac{-C}{A+B} \leq \varepsilon.
\end{align*}

Therefore, we can conclude
\begin{align*}
    \frac{|\Pr(\ry=y,\hat{\ry}=y|\rz=0) - \Pr(\ry=y,\hat{\ry}=y|\rz=1)|}{2 \Pr(\ry=y|\hat{\ry}=y, \rz=0)+\Pr(\hat{\ry}=y|\rz=0)+\Pr(\ry=y|\hat{\ry}=y, \rz=1)} \leq \varepsilon.
\end{align*}

If we reverse the $\rz$ values in the derivation, we get the same formula with only the $\rz$ value changed in the above expression.
As a result, we get the following inequality, which is in Proposition~\ref{prop:ppdp}:
\begin{align*}
    \frac{|\Pr(\ry=y,\hat{\ry}=y|\rz=z) - \Pr(\ry=y,\hat{\ry}=y|\rz=z')|}{2 \Pr(\ry=y|\hat{\ry}=y, \rz=z)+\Pr(\hat{\ry}=y|\rz=z)+\Pr(\ry=y|\hat{\ry}=y, \rz=z')} \leq \varepsilon, ~~
     z\neq z',~z,z' \in \mathcal{\sZ},~y \in \mathcal{\sY}.    
\end{align*}

\end{proof}

Therefore, to make $\varepsilon = 0$, the numerator term $|\Pr(\ry=y,\hat{\ry}=y|\rz=z) - \Pr(\ry=y,\hat{\ry}=y|\rz=z')|$ should be zero. When $|\Pr(\ry=y,\hat{\ry}=y|\rz=z) - \Pr(\ry=y,\hat{\ry}=y|\rz=z')| = 0$, the following is satisfied: $(\ry, \hat{\ry}) \perp \rz$, which implies $\ry \perp \rz$ and $\hat{\ry} \perp \rz$. Here, $\ry \perp \rz$ indicates that $|\Pr(\ry=y|\rz=z)-\Pr(\ry=y|\rz=z')|$ is zero. 
As a result, perfectly satisfying PP and DP requires the data to be fully unbiased. We thus suspect that the achievable model fairness w.r.t. both PP and DP is affected by the ($\ry, \rz$)-correlation.

\subsection{$\varepsilon$-EO \& $\varepsilon$-PP}
\label{appendix:eo_pp}

Continuing from Sec.~\ref{sec:fundamentallimits}, we consider fair training w.r.t. both equalized odds (EO) and predictive parity (PP). 
We can define $\varepsilon$-EO \& $\varepsilon$-PP as follows:
\begin{align*}
    -\varepsilon \leq \Pr(\hat{\ry}=y|\ry=y, \rz=0) &- \Pr(\hat{\ry}=y|\ry=y, \rz=1) \leq \varepsilon ~~\text{(i.e., $\varepsilon$-EO)}\\
    -\varepsilon \leq \Pr(\ry=y|\hat{\ry}=y, \rz=0) &- \Pr(\ry=y|\hat{\ry}=y, \rz=1) \leq \varepsilon ~~\text{(i.e., $\varepsilon$-PP)}
\end{align*}

Based on these definitions, we give a proposition for $\varepsilon$-EO and $\varepsilon$-PP:
\begin{proposition}[$\varepsilon$-EO \& $\varepsilon$-PP]\label{prop:eopp}
Let a model achieve both $\varepsilon$-EO and $\varepsilon$-PP. Then, the following inequality holds:
\begin{align*}
    \frac{\Pr(\hat{\ry}{=}y|\ry{=}y,\rz{=}z')}{\Pr(\hat{\ry}{=}y, \rz{=}z) + \Pr(\ry{=}y, \rz{=}z)} \cdot | \Pr(\ry{=}y)-\Pr(\ry{=}y | \rz{=}z') \cdot \frac{\Pr(\hat{\ry}{=}y)}{\Pr(\hat{\ry}{=}y | \rz{=}z')} |
    \leq \varepsilon , ~~
     z\neq z',~z,z' \in \mathcal{\sZ},~y \in \mathcal{\sY}.
\end{align*}
\end{proposition}

\begin{proof}
From $\varepsilon$-EO, we get 
\begin{align}
    -\varepsilon + \Pr(\hat{\ry}=y|\ry=y, \rz=1) \leq \Pr(\hat{\ry}=y|\ry=y, \rz=0) \leq \varepsilon + \Pr(\hat{\ry}=y|\ry=y, \rz=1). \label{eq:eopp_1} 
\end{align}

From $\varepsilon$-PP, we get
\begin{align}
    -\varepsilon \leq \Pr(\hat{\ry}=y|\ry=y, \rz=0) \frac{\Pr(\ry=y, \rz=0)}{\Pr(\hat{\ry}=y, \rz=0)} - \Pr(\hat{\ry}=y|\ry=y, \rz=1) \frac{\Pr(\ry=y, \rz=1)}{\Pr(\hat{\ry}=y, \rz=1)} \leq \varepsilon. \label{eq:eopp_2}
\end{align}

By substituting Eq.~\ref{eq:eopp_1} to Eq.~\ref{eq:eopp_2}, we can get the following inequality:
\begin{align}
    &(-\varepsilon+\Pr(\hat{\ry}=y|\ry=y, \rz=1)) \frac{\Pr(\ry=y, \rz=0)}{\Pr(\hat{\ry}=y, \rz=0)} - \Pr(\hat{\ry}=y|\ry=y, \rz=1) \frac{\Pr(\ry=y, \rz=1)}{\Pr(\hat{\ry}=y, \rz=1)} \nonumber \\
    &\leq \Pr(\hat{\ry}=y|\ry=y, \rz=0) \frac{\Pr(\ry=y, \rz=0)}{\Pr(\hat{\ry}=y, \rz=0)} - \Pr(\hat{\ry}=y|\ry=y, \rz=1) \frac{\Pr(\ry=y, \rz=1)}{\Pr(\hat{\ry}=y, \rz=1)} \label{eq:eopp_3}\\
    &\leq (\varepsilon+\Pr(\hat{\ry}=y|\ry=y, \rz=1)) \frac{\Pr(\ry=y, \rz=0)}{\Pr(\hat{\ry}=y, \rz=0)} - \Pr(\hat{\ry}=y|\ry=y, \rz=1) \frac{\Pr(\ry=y, \rz=1)}{\Pr(\hat{\ry}=y, \rz=1)}. \nonumber 
\end{align}

By subtracting Eq.~\ref{eq:eopp_2} from Eq.~\ref{eq:eopp_3},
\begin{align*}
    &-\varepsilon+(-\varepsilon+\Pr(\hat{\ry}=y|\ry=y, \rz=1)) \frac{\Pr(\ry=y, \rz=0)}{\Pr(\hat{\ry}=y, \rz=0)} - \Pr(\hat{\ry}=y|\ry=y, \rz=1) \frac{\Pr(\ry=y, \rz=1)}{\Pr(\hat{\ry}=y, \rz=1)} \\
    &\leq 0 \\
    &\leq \varepsilon + (\varepsilon+\Pr(\hat{\ry}=y|\ry=y, \rz=1)) \frac{\Pr(\ry=y, \rz=0)}{\Pr(\hat{\ry}=y, \rz=0)} - \Pr(\hat{\ry}=y|\ry=y, \rz=1) \frac{\Pr(\ry=y, \rz=1)}{\Pr(\hat{\ry}=y, \rz=1)}.
\end{align*}

By rearranging the terms, we get the following inequality:
\begin{align*}
    &-\varepsilon \cdot \frac{\Pr(\hat{\ry}=y, \rz=0) + \Pr(\ry=y, \rz=0)}{\Pr(\hat{\ry}=y, \rz=0)}\\
    &\leq \Pr(\hat{\ry}=y|\ry=y,\rz=1) \{ \frac{\Pr(\ry=y, \rz=0)}{\Pr(\hat{\ry}=y, \rz=0)} - \frac{\Pr(\ry=y, \rz=1)}{\Pr(\hat{\ry}=y, \rz=1)} \} \\
    &\leq \varepsilon \cdot \frac{\Pr(\hat{\ry}=y, \rz=0) + \Pr(\ry=y, \rz=0)}{\Pr(\hat{\ry}=y, \rz=0)}.
\end{align*}

By multiplying $\frac{\Pr(\hat{\ry}=y, \rz=0)}{\Pr(\hat{\ry}=y, \rz=0) + \Pr(\ry=y, \rz=0)}$ for all terms in the inequality, we can get
\begin{align*}
    -\varepsilon \leq \frac{\Pr(\hat{\ry}=y, \rz=0)\Pr(\hat{\ry}=y|\ry=y,\rz=1)}{\Pr(\hat{\ry}=y, \rz=0) + \Pr(\ry=y, \rz=0)} \{ \frac{\Pr(\ry=y, \rz=0)}{\Pr(\hat{\ry}=y, \rz=0)} - \frac{\Pr(\ry=y, \rz=1)}{\Pr(\hat{\ry}=y, \rz=1)} \}
    \leq \varepsilon 
\end{align*}

which is rewritten as follows:
\begin{align*}
    -\varepsilon \leq \frac{\Pr(\hat{\ry}=y|\ry=y,\rz=1)}{\Pr(\hat{\ry}=y, \rz=0) + \Pr(\ry=y, \rz=0)} \{ \Pr(\ry=y, \rz=0) - \frac{\Pr(\hat{\ry}=y, \rz=0)}{\Pr(\hat{\ry}=y, \rz=1)} \cdot \Pr(\ry=y, \rz=1) \}
    \leq \varepsilon.
\end{align*}

Since $\Pr(\ry=y, \rz=0) = \Pr(\ry=y)-\Pr(\ry=y, \rz=1)$ by the total probability law,
\begin{align*}
    -\varepsilon \leq \frac{\Pr(\hat{\ry}{=}y|\ry{=}y,\rz{=}1)}{\Pr(\hat{\ry}{=}y, \rz{=}0) + \Pr(\ry{=}y, \rz{=}0)} \{ \Pr(\ry{=}y)-\Pr(\ry{=}y, \rz{=}1) - \frac{\Pr(\hat{\ry}{=}y, \rz{=}0)}{\Pr(\hat{\ry}{=}y, \rz{=}1)} \cdot \Pr(\ry{=}y, \rz{=}1) \}
    \leq \varepsilon.
\end{align*}

By arranging the terms,
\begin{gather*}
    -\varepsilon \leq \frac{\Pr(\hat{\ry}=y|\ry=y,\rz=1)}{\Pr(\hat{\ry}=y, \rz=0) + \Pr(\ry=y, \rz=0)} \{ \Pr(\ry=y)-\Pr(\ry=y, \rz=1) \cdot ( 1+ \frac{\Pr(\hat{\ry}=y, \rz=0)}{\Pr(\hat{\ry}=y, \rz=1)}) \}
    \leq \varepsilon \\
    \Rightarrow -\varepsilon \leq \frac{\Pr(\hat{\ry}=y|\ry=y,\rz=1)}{\Pr(\hat{\ry}=y, \rz=0) + \Pr(\ry=y, \rz=0)} \{ \Pr(\ry=y)-\Pr(\ry=y, \rz=1) \cdot \frac{\Pr(\hat{\ry}=y)}{\Pr(\hat{\ry}=y, \rz=1)} \}
    \leq \varepsilon.
\end{gather*}

Therefore, we can conclude
\begin{align*} 
    \frac{\Pr(\hat{\ry}=y|\ry=y,\rz=1)}{\Pr(\hat{\ry}=y, \rz=0) + \Pr(\ry=y, \rz=0)} \cdot | \Pr(\ry=y)-\Pr(\ry=y, \rz=1) \cdot \frac{\Pr(\hat{\ry}=y)}{\Pr(\hat{\ry}=y, \rz=1)} |
    \leq \varepsilon.
\end{align*}

If we reverse the $\rz$ values in the derivation, we get the same formula with only the $\rz$ value changed in the above expression.
As a result, we get the following inequality, which is in Proposition~\ref{prop:eopp}:
\begin{align*}
     \frac{\Pr(\hat{\ry}{=}y|\ry{=}y,\rz{=}z')}{\Pr(\hat{\ry}{=}y, \rz{=}z) + \Pr(\ry{=}y, \rz{=}z)} \cdot | \Pr(\ry{=}y)-\Pr(\ry{=}y | \rz{=}z') \cdot \frac{\Pr(\hat{\ry}{=}y)}{\Pr(\hat{\ry}{=}y | \rz{=}z')} |
    \leq \varepsilon , ~~
     z\neq z',~z,z' \in \mathcal{\sZ},~y \in \mathcal{\sY}.
\end{align*}

\end{proof}

Therefore, $\varepsilon = 0$ when $\Pr(\ry=y) = \Pr(\ry=y | \rz=z')$ and $\Pr(\hat{\ry}=y) = \Pr(\hat{\ry}=y | \rz=z')$ (i.e., $\ry \perp \rz$ and $\hat{\ry} \perp \rz$). Here, satisfying both $\ry \perp \rz$ and $\hat{\ry} \perp \rz$ is the sufficient condition of perfectly satisfying EO and PP. Note that $\ry \perp \rz$ implies that the data is unbiased.
We thus suspect that the achievable model fairness w.r.t. both EO and PP is affected by the ($\ry, \rz$)-correlation. 

\newpage
\subsection{\revisionnew{Estimating Shifted Correlation Range Using Samples}}
\label{appendix:estimate_corr}
\vspace{0.1cm}

\revisionnew{Continuing from Sec.~\ref{sec:optimization}, we discuss using samples in the deployment data to estimate the shifted correlation value $c$ and its range $[\alpha, \beta]$. To this end, we show the relationship between the number of samples and the confidence of the estimation through the following theorem.
}

\revisionnew{
\begin{theorem}\label{thm:correlation}
    Let $n_1$, $n_0$, $n_{11}$, and $n_{01}$ be the number of samples with $(\rz=1)$, $(\rz=0)$, $(\rz=1, \ry=1)$, and $(\rz=0, \ry=1)$, respectively. Let $\hat{c}$ = $\dfrac{n_{11}}{n_1}-\dfrac{n_{01}}{n_0}$. If $n_1 \geq \dfrac{2 \ln (4/\delta)}{\eps^2}$ and $n_0 \geq \dfrac{2 \ln (4/\delta)}{\eps^2}$, then $|\hat{c}-c| \leq \eps$ with probability $1-\delta$. 
\end{theorem}
}

\begin{proof}
\revisionnew{By Hoeffding's inequality~\citep{10.2307/2282952}, we can write the following inequality w.r.t.\@ $\frac{n_{11}}{n_1}$:
\begin{align}
\Pr(\left|\frac{n_{11}}{n_1}-\Pr(\ry=1|\rz=1)\right|\geq \frac{\eps}{2}) \leq 2 \exp(-\frac{n_1 \eps^2}{2}).
\end{align}
Using the condition w.r.t.\@ $n_1$ in Theorem~\ref{thm:correlation} (i.e., $n_1 \geq \frac{2 \ln (4/\delta)}{\eps^2}$), we can rewrite the Hoeffding's inequality as follows:
\begin{align}
\Pr(\left|\frac{n_{11}}{n_1}-\Pr(\ry=1|\rz=1)\right|\geq \frac{\eps}{2}) &\leq 2 \exp(-\frac{n_1 \eps^2}{2})\\
&\leq 2 \exp(-\frac{2 \ln (4/\delta)}{\eps^2}\frac{\eps^2}{2})\\
&= 2 \exp(\ln (\delta/4))\\
&= \frac{\delta}{2}.
\end{align}
Similarly, the inequality w.r.t.\@ $\frac{n_{01}}{n_0}$ can be written as follows:
\begin{align}
\Pr(\left|\frac{n_{01}}{n_0}-\Pr(\ry=1|\rz=0)\right|\geq \frac{\eps}{2}) \leq 2 \exp(-\frac{n_0 \eps^2}{2}) \leq \frac{\delta}{2}.
\end{align}
These inequalities show that if $n_1 \geq \dfrac{2 \ln (4/\delta)}{\eps^2}$ and $n_0 \geq \dfrac{2 \ln (4/\delta)}{\eps^2}$, then $|\dfrac{n_{11}}{n_1}-\Pr(\ry{=}1|\rz{=}1)| \leq \eps/2$ and $|\dfrac{n_{01}}{n_0}-\Pr(\ry{=}1|\rz{=}0)| \leq \eps/2$ each with probability at least $1-\delta/2$. }

\revisionnew{Therefore, if $n_1 \geq \dfrac{2 \ln (4/\delta)}{\eps^2}$ and $n_0 \geq \dfrac{2 \ln (4/\delta)}{\eps^2}$, then $|\hat{c}-c| = |(\dfrac{n_{11}}{n_1}-\dfrac{n_{01}}{n_0})-(\Pr(\ry{=}1|\rz{=}1)-\Pr(\ry{=}1|\rz{=}0))|\leq\eps/2 + \eps/2 \leq \eps$ with probability at least $1-\delta$.}
\end{proof}

\revisionnew{By using Theorem~\ref{thm:correlation}, we can calculate how many samples are required to achieve a specific error range with a specific probability. For example, if we aim to have an error range $c-0.1\leq \hat{c} \leq c+0.1$ with 0.95 probability (i.e., $\eps=0.1$ and $\delta=0.05$), we need $n_1 \geq 876$ and $n_0 \geq 876$ samples to satisfy the condition $n \geq \dfrac{2 \ln (4/\delta)}{\eps^2}$. Also, when we measure $\hat{c}$ using $m$ ($=n_1+n_0$) samples, we can construct the shift range $[\alpha, \beta]$ with a specific probability $1-\delta$  -- see details in the following section (Sec.~\ref{appendix:mle}).}

\vspace{0.3cm}
\subsection{\revisionnew{Maximum Likelihood Estimation in Our Scenario}}
\label{appendix:mle}
\vspace{0.2cm}

\revisionnew{Continuing from Sec.~\ref{sec:optimization}, we explain the details of the maximum likelihood estimation (MLE) in Algorithm~\ref{alg:overall}. During the estimation, we use Theorem~\ref{thm:correlation} in Sec.~\ref{appendix:estimate_corr}, which explains how to get the confidence interval of the estimation using $m$ samples with a specific probability $(1-\delta)$. Note that the confidence level $(1-\delta)$ is a hyperparameter (e.g., $(1-\delta) = 0.9$). }

\revisionnew{Algorithm~\ref{alg:mle} shows the overall steps in the estimation. We first calculate the estimated correlation constant $c$ (i.e., $\hat{c}$). Then, we find the error bound $\eps$ by using the conditions in Theorem~\ref{thm:correlation} (i.e., $\eps = \sqrt{\frac{2}{n}\ln(\frac{4}{\delta})}$). As a result, we can construct the shift range $[\alpha, \beta] = [\hat{c}-\eps, \hat{c}+\eps]$ with a probability $1-\delta$.}

\vspace{0.2cm}
\setlength{\textfloatsep}{10pt}
\begin{algorithm}[h]
\DontPrintSemicolon
\setstretch{1.0}
    \SetKwInput{Input}{Input}
    \SetKwInOut{Output}{Output}
    \SetNoFillComment
    
    \Input{$m$ samples of ($\ry_\text{deploy}, \rz_\text{deploy}$), confidence level of estimation $1-\delta$}
    
    $\hat{\Pr}(\ry=y|\rz=z) =\dfrac{(\text{num.\@ of samples with ($\ry=y, \ry=z$)})}{(\text{num.\@ of samples with ($\ry=z$)})}$
    
    $\hat{c} = \dfrac{\hat{\Pr}(\ry=1|\rz=1)}{\hat{\Pr}(\rz=1)} - \dfrac{\hat{\Pr}(\ry=1|\rz=0)}{\hat{\Pr}(\rz=0)}$
    
    $n_1 = \text{num.\@ of samples with ($\rz=1$)}$
    
    $n_0 = \text{num.\@ of samples with ($\rz=0$)}$
    
    $n = \min(n_1, n_0)$ 
    
    $\eps = \sqrt{\dfrac{2}{n}\ln(\dfrac{4}{\delta})}$
    
    $\alpha = \hat{c} - \eps$
    
    $\beta = \hat{c} + \eps$
    
    \Output{$\alpha$, $\beta$}
    \caption{MLE}
    \label{alg:mle}
\end{algorithm}

\vspace{0.3cm}
\subsection{Semidefinite Relaxation}
\label{appendix:sdprelaxation}
\vspace{0.2cm}

Continuing from Sec.~\ref{sec:optimization}, we provide details of the semidefinite relaxation in our optimization.

Recall our original optimization as follows:
\begin{gather*}
    \underset{w'}{\min} ~~ \underset{\forall y, z}{\sum} (w_{\ry=y, \rz=z}-w'_{\ry=y, \rz=z})^2\\
    \text{s.t. } \alpha \leq \frac{w'_{\ry=1, \rz=1}}{w'_{\ry=1, \rz=1} + w'_{\ry=0, \rz=1}} - \frac{w'_{\ry=1, \rz=0}}{w'_{\ry=1, \rz=0} + w'_{\ry=0, \rz=0}} \leq \beta,\\
    |(w'_{\ry=1, \rz=1} + w'_{\ry=1, \rz=0})-{\Pr}_{\text{train}}(\ry=1)| \leq \gamma_\ry,\\
    |(w'_{\ry=1, \rz=1} + w'_{\ry=0, \rz=1})-{\Pr}_{\text{train}}(\rz=1)| \leq \gamma_\rz,\\
    \underset{\forall y, z}{\sum} w'_{\ry=y, \rz=z} = 1,~~
    0 \leq w'_{\ry=y, \rz=z} \leq 1,~~\forall y \in \{0, 1\}, z \in \{0, 1\}
\end{gather*}
where ${\Pr}_{\text{train}}(\ry=1) {=} w_{\ry=1, \rz=1} {+} w_{\ry=1, \rz=0}$ and ${\Pr}_{\text{train}}(\rz=1) = w_{\ry=1, \rz=1} {+} w_{\ry=0, \rz=1}$.

As the above optimization is a non-convex quadratically constrained quadratic problem (non-convex QCQP), we now apply the semidefinite relaxation (SDP relaxation)~\citep{park2017general}.
We first rewrite the above optimization using matrices:
\begin{gather*}
    \underset{x}{\min}~~ x^T P_0 x+q_0^Tx\\
    \text{s.t.}~~  x^T P_\alpha x \geq 0,~ x^T P_\beta x \leq 0,\\
     |q_2^T x - {\Pr}_{\text{train}}(y=1)| \leq \gamma_\ry,\\
     |q_3^T x - {\Pr}_{\text{train}}(z=1)| \leq \gamma_\rz\\
     q_4^Tx = 1, ~~  0 \leq x_i \leq 1~~\forall i
\end{gather*}
where $x = \begin{bmatrix} x_1 & x_2 & x_3 & x_4 \end{bmatrix}^T = \begin{bmatrix} w'_{1, 1} & w'_{1, 0} & w'_{0, 1} & w'_{0, 0} \end{bmatrix}^T$,
$q_0 = -2\begin{bmatrix} w_{1, 1} & w_{1, 0} & w_{0, 1} & w_{0, 0} \end{bmatrix}^T$,
$q_2 = \begin{bmatrix} 1 & 0 & 1 & 0 \end{bmatrix}^T$, $q_3 = \begin{bmatrix} 1 & 1 & 0 & 0 \end{bmatrix}^T$, $q_4 = \begin{bmatrix} 1 & 1 & 1 & 1 \end{bmatrix}^T$, $P_0 = \diag(\mathbf{1})~$, 
\resizebox{0.85\hsize}{!}{\textcolor{black}{$P_\alpha = \begin{bsmallmatrix} 0 & {-\alpha}/{2} & 0 & {(1-\alpha)}/{2}\\ {-\alpha}/{2} & 0 & {(-1-\alpha)}/{2} & 0\\ 0 & {(-1-\alpha)}/{2} & 0 & {-\alpha}/{2}\\ {(1-\alpha)}/{2} & 0 & {-\alpha}/{2} & 0 \end{bsmallmatrix}$, and $P_\beta = \begin{bsmallmatrix} 0 & {-\beta}/{2} & 0 & {(1-\beta)}/{2}\\ {-\beta}/{2} & 0 & {(-1-\beta)}/{2} & 0\\ 0 & {(-1-\beta)}/{2} & 0 & {-\beta}/{2}\\ {(1-\beta)}/{2} & 0 & {-\beta}/{2} & 0 \end{bsmallmatrix}$}.}

As $x^TPx = \textbf{Tr}(P(xx^T))$, we can get the following optimization:
\begin{gather*}
    \underset{X, x}{\min}~~ \textbf{Tr}(XP_0)+q_0^Tx\\
    \text{s.t.}~~  \textbf{Tr}(XP_\alpha) \geq 0,~ \textbf{Tr}(XP_\beta) \leq 0\\
     |q_2^T x - {\Pr}_{\text{train}}(y=1)| \leq \gamma_\ry,\\
     |q_3^T x - {\Pr}_{\text{train}}(z=1)| \leq \gamma_\rz\\
     q_4^Tx = 1, ~~  0 \leq x_i \leq 1~~\forall i,\\
     X = xx^T
\end{gather*}
where $\textbf{Tr}(\cdot)$ is the trace and $~X=xx^T$.

In the above optimization, the last constraint is non-convex. Thus, we relax the non-convex constraint $X = xx^T$ into $X-xx^T \succeq 0$, which is convex, and then use a Schur complement~\citep{zhang2006schur} to get the final SDP form:
\begin{gather*}
    \underset{X, x}{\min}~~ \textbf{Tr}(XP_0)+q_0^Tx\\
    \text{s.t.}~~  
    \textbf{Tr}(XP_\alpha) \geq 0,~ \textbf{Tr}(XP_\beta) \leq 0\\
     |q_2^T x - {\Pr}_{\text{train}}(y=1)| \leq \gamma_\ry,\\
     |q_3^T x - {\Pr}_{\text{train}}(z=1)| \leq \gamma_\rz\\
     q_4^Tx = 1, ~~  0 \leq x_i \leq 1~~\forall i,\\
     \begin{bmatrix} X & x \\ x^T & 1 \end{bmatrix} \succeq 0.
\end{gather*}

When we solve the above SDP relaxation problem using convex optimization solvers, we set the 5$\times$5-matrix A as the variable so as to indicate $\begin{bmatrix} X & x \\ x^T & 1 \end{bmatrix}$.
Then, we get the solution $x = \begin{bmatrix} x_1 & x_2 & x_3 & x_4 \end{bmatrix}^T$ by taking the first four elements of the last vector in the resulting matrix A.

\vspace{1.5cm}
\subsection{Optional Step for Minimizing Overall Distribution Change}
\label{appendix:min_dist_change}
\vspace{0.1cm}

Continuing from Sec.~\ref{sec:algorithm}, we explain the details on the optional step in Algorithm~\ref{alg:overall}.
We can use the optional step ($\texttt{MinDistChange}$) to find non-uniform data sample weights within each ($\ry, \rz$)-class that minimizes the overall distribution change in terms of the Wasserstein distance~\citep{givens1984class}, rather than using identical weights.
In the optional step, we first divide the examples in each ($\ry$=$y$, $\rz$=$z$)-class into two sets $\{\ry$=$y$, $\rz$=$z$, $\rt$=$0\}$ and $\{\ry$=$y$, $\rz$=$z$, $\rt$=$1\}$ with equal numbers, where $\rt$ is a specific feature that can be a criterion for dividing the examples. For example, we can find a median of a non-sensitive attribute $\rx_1$ and set $\rt = 0$ for an example if the $\rx_1$ value of the example is lower than the median of $\rx_1$. We set $\rt = 1$ otherwise.
We then make a candidate set of partial weights $w'_{\ry=y, \rz=z, \rt=t}$ on each ($\ry$=$y$, $\rz$=$z$, $\rt$=$t$)-class, where $\sum w'_{\ry=y, \rz=z, \rt=t} = w'_{\ry=y, \rz=z}$. Note that we can extend $\rt$ beyond the binary setting to find more detailed partial weights.
For each partial weight candidate, we calculate the candidate sample-wise weights ($\textbf{d}_{\text{tmp}}$) to ensure that the sample weight sum in each ($\ry$=$y$, $\rz$=$z$, $\rt$=$t$)-class is $w_{\ry=y, \rz=z, \rt=t} \cdot n$, where $n$ is the total number of samples in the original training data. 
Then, we draw new data $\tilde{D}$ from the original training data $D$ via weighted sampling according to $\textbf{d}_{\text{tmp}}$. 
Then, we calculate the Wasserstein distance between $\tilde{D}$ and $D$ via an optimal transport technique~\citep{peyre2019computational}. As a result, the algorithm returns the final sample-wise weights ($\textbf{d}_{\text{min}}$) that result in the closest distribution from the original data. 

\vspace{2cm}
\setlength{\textfloatsep}{10pt}
\begin{algorithm}[h]
\DontPrintSemicolon
\setstretch{1.2}
    \SetKwInput{Input}{Input}
    \SetKwInOut{Output}{Output}
    \SetNoFillComment
    
    \Input{train data $D$, original ratio $w_{\ry, \rz}$, new ratio $w'_{\ry, \rz}$}
    
    In each ($\ry$=$y$, $\rz$=$z$)-class, divide the examples into two sets $\{\ry$=$y$, $\rz$=$z$, $\rt$=$0\}$ and $\{\ry$=$y$, $\rz$=$z$, $\rt$=$1\}$ with equal numbers\\
    $\text{partials}$ = $[0, 1/m, 2/m, ..., 1]$\\
    $\text{W}_{\ry=y, \rz=z}$ $\gets$ [], ~ $\forall (y, z)\in \mathcal{\sY} \times \mathcal{\sZ}$\\
    \For{each ($\ry$=$y$, $\rz$=$z$)-class}{
        \For{\normalfont{p} in \normalfont{partials}}{
            $w'_{\ry=y, \rz=z, \rt=0} = w'_{\ry=y, \rz=z} \cdot p$\\
            $w'_{\ry=y, \rz=z, \rt=1} = w'_{\ry=y, \rz=z} \cdot (1-p)$\\
            Append $[w'_{\ry=y, \rz=z, \rt=0}, w'_{\ry=y, \rz=z, \rt=1}]$ to $\text{W}_{\ry=y, \rz=z}$
        }
    }
    
    $\text{candidates} \gets $ $\text{W}_{\ry=1, \rz=1} \times \text{W}_{\ry=1, \rz=0} \times \text{W}_{\ry=0, \rz=1} \times \text{W}_{\ry=0, \rz=0}$\\
    minD $\gets$ an initial large value\\
    \For{\normalfont{partial-weight} in \normalfont{candidates}}{
        $d_j$ $\gets$ $w'_{\ry=y, \rz=z, \rt=t}/(w_{\ry=y, \rz=z} \cdot 0.5), \forall j \in \mathbb{I}_{(y,z,t)}, \forall (y, z, t)\in \mathcal{\sY} \times \mathcal{\sZ} \times \{0, 1\}$\\
        $\textbf{d}_{\text{tmp}}$ = $\{d_i\}_{ i=1,...,n}$\\
        Draw data $\tilde{D}$ from $D$ via weighted sampling w.r.t. $\textbf{d}_{\text{tmp}}$\\
        wassD $\gets$ Calculate the Wasserstein distance between $\tilde{D}$ and $D$ via optimal transport\\
        \If{wassD $<$ minD}{
            $\textbf{d}_{\text{min}}$ = $\textbf{d}_{\text{tmp}}$\\
            minD = wassD
        }
    }
    \Output{$\textbf{d}_{\text{min}}$}
    \caption{MinDistChange}
    \label{alg:min_dist_change}
\end{algorithm}

\vspace*{3cm}
\section{Appendix -- Experiments}
\label{appendix:experiments}

\subsection{Simulation Settings}
\label{appendix:simulation_setting}

Continuing from Sec.~\ref{sec:fundamentallimits}, we explain the details of the simulation. 
The goal of the simulation is to confirm our theoretical observations by generating various synthetic classifiers that show the full range of possible model performances. 
To this end, we generate 1000 data samples, where each sample has $\ry$ and $\rz$ features. We first set $\ry$ and $\rz$ to have a specific correlation.
\revision{We then generate various synthetic classifiers to have different predicted labels $\hat{\ry}$ by varying the probability $\Pr(\hat{\ry}=1)$ in each ($\ry$, $\rz$)-class, regardless of the other input features. Note that we do not train actual models (e.g., logistic regression, SVM).}
For each synthetic classifier, we measure the accuracy and fairness performances based on $\ry$, $\rz$, and the classifier's $\hat{\ry}$.
We repeat the above procedures while varying $\ry$ and $\rz$ to have different correlations.
As a result, we get the simulation results in Figure~\ref{fig:simulation} by plotting all the classifier accuracy and fairness performances within each ($\ry, \rz$)-correlation.

\newpage
\subsection{Empirical Observations for Equalized Odds}
\label{appendix:eo_empirical}

Continuing from Sec.~\ref{sec:fundamentallimits}, we perform an experiment to observe that higher ($\ry$, $\rz$)-correlation leads to a worse accuracy-fairness tradeoff w.r.t. equalized odds (EO) in practice.
Figure~\ref{fig:simulation_eo} shows the accuracy-unfairness performances of the logistic regression model trained using FairBatch~\citep{roh2021fairbatch} to improve EO on two synthetic datasets with low and high ($\ry$, $\rz$)-correlations.
As a result, low correlation enables classifiers to attain better accuracy-fairness tradeoffs (i.e., close to the bottom right) w.r.t. EO.
This experiment shows that the ($\ry, \rz$)-correlation indeed affects the accuracy-fairness tradeoff w.r.t. EO in practice.

\begin{figure}[h]
\centering
\includegraphics[width=0.5\columnwidth,trim=0cm 0.3cm 0cm 0cm]{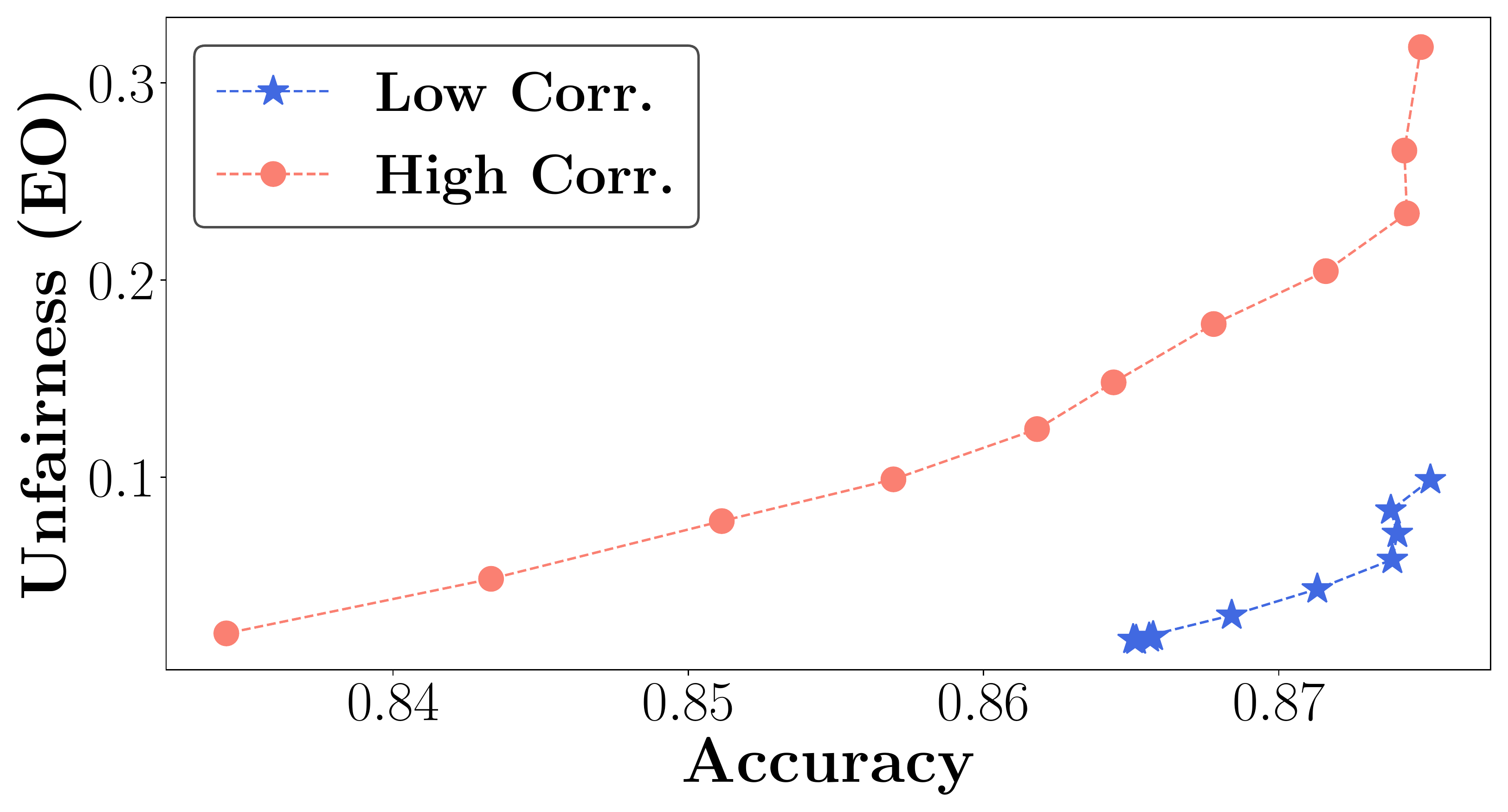}
\caption{Accuracy-unfairness performances of fair training with FairBatch~\citep{roh2021fairbatch} on two synthetic data with different ($\ry, \rz$)-correlations. We measure fairness w.r.t. equalized odds (EO).}
\vspace{0.5cm}
\label{fig:simulation_eo}
\end{figure}



\subsection{Other Experimental Settings}
\label{appendix:ex_setting}

Continuing from Sec.~\ref{sec:experiments}, we provide more details on the experimental settings. The batch sizes of the synthetic, COMPAS, and AdultCensus datasets are 100, 200, and 2,000, respectively. For the synthetic dataset, we use 2,000 samples for the training dataset and 1,000 samples for the test dataset. 
For the real datasets, we split the entire data into 4:1 for the training and test datasets. 
We set the learning rate to 0.0005. 
\textcolor{black}{Our experiments are performed using PyTorch on a Linux server with Intel Xeon Silver 4210R CPUs and NVIDIA Quadro RTX 8000 GPUs.}

\textcolor{black}{We generate the synthetic training dataset with 2,000 samples and consists of two non-sensitive attributes $(\rx_1, \rx_2)$, one sensitive attribute $\rz$, and one label attribute $\ry$. Each sample $(\rx_1, \rx_2, \ry)$ is drawn from the following Gaussian distributions: $(\rx_1,\rx_2)|\ry=0 \sim \mathcal{N}([-2; -2], [10, 1; 1, 3])$ and $(\rx_1,\rx_2)|\ry=1 \sim \mathcal{N}([2; 2], [5, 1; 1, 5])$. For the sensitive attribute $\rz$, we generate a biased distribution: $\Pr(\rz=1)=\Pr((\rx'_1, \rx'_2)|\ry=1)/[\Pr((\rx'_1, \rx'_2)|\ry=0)+\Pr((\rx'_1, \rx'_2)|\ry=1)]$ where $(\rx'_1, \rx'_2)=(\rx_1\cos(\pi/4)-\rx_2\sin(\pi/4), \rx_1\sin(\pi/4)+\rx_2\cos(\pi/4))$.}

When we construct the synthetic test dataset by modifying the $\rz$ values while fixing the $\rx$ and $\ry$ distributions in the original test data, we change the $k$ value in $(\rx'_1, \rx'_2)=(\rx_1\cos(\pi/k)-\rx_2\sin(\pi/k), \rx_1\sin(\pi/k)+\rx_2\cos(\pi/k))$.

\yuji{For all in-processing algorithms, we start from a candidate set and use cross-validation on the (pre-processed) training data to choose the hyperparameters that result in the best fairness while having an accuracy that best aligns with other results.}

To support multiple fairness metrics, we naturally extend each in-processing approach by combining the fairness constraints for different metrics via a tuning knob that adjusts the importance between each metric. Here, the fairness constraints are implemented differently in each algorithm. Fairness Constraints (FC)~\citep{DBLP:conf/aistats/ZafarVGG17} adds each unfairness penalty term to the loss function. Thus, we extend FC by adding multiple penalty terms and adjust the importance between each penalty term by a tuning knob. Adversarial Debiasing (AD)~\citep{DBLP:conf/aies/ZhangLM18} utilizes a discriminator of each fairness metric for the adversarial training. Thus, we extend AD by adding multiple fairness discriminators and adjust the importance between each discriminator by adding a tuning knob in the classifier's loss function. FairBatch (FB)~\citep{roh2021fairbatch} solves a bilevel optimization that has an objective for minimizing a fairness disparity according to each fairness metric. Thus, we extend FB by minimizing the maximum of fairness disparities (e.g., $\max(DP disp., EO disp.)$) and adjust the importance between each disparity using a tuning knob. 

\subsection{Accuracy and Fairness -- AdultCensus}
\label{appendix:adultcensus}

Continuing from Sec.~\ref{sec:accfair}, we show the results on the AdultCensus dataset. Table~\ref{tbl:adult} shows the accuracy and fairness performances of the algorithms on the AdultCensus test dataset w.r.t. a single metric (DP) and multiple metrics (DP \& EO). Other setting are identical to Table~\ref{tbl:synthetic}.
We observe consistent results where our framework improves accuracy and fairness of the in-processing-only baselines and also shows better fairness than the two-step baselines using RW.

\begin{table}[h]
  \caption{Performances on the AdultCensus test dataset. Other experimental settings are identical to those in Table~\ref{tbl:synthetic}.}
  \label{tbl:adult}
  \centering
\scalebox{0.9}{
  \begin{tabular}{l@{\hspace{7pt}}c@{\hspace{7pt}}c@{\hspace{12pt}}c@{\hspace{7pt}}c}
    \toprule
      & \multicolumn{2}{c}{Single (DP)} & \multicolumn{2}{c}{Multiple (DP \& EO)} \\
    \cmidrule(r){1-5}
      Method & Acc. & Unfair. & Acc. & Unfair. \\
    \midrule
    LR & .824 $\pm$ .001 & .074 $\pm$ .002 & .824 $\pm$ .001 & .074 $\pm$ .002 \\
    \cmidrule(l){1-5}
    FC & .805 $\pm$ .013 & .021 $\pm$ .004 & .807 $\pm$ .013 & .078 $\pm$ .017 \\
    RW+FC & .810 $\pm$ .004 & .020 $\pm$ .006 & .810 $\pm$ .008 & .048 $\pm$ .013 \\
    \textbf{Ours}+FC & .820 $\pm$ .003 & \textbf{.005 $\pm$ .002} & .805 $\pm$ .014 & \textbf{.033 $\pm$ .013} \\
    \cmidrule(l){1-5}
    AD & .777 $\pm$ .022 & \textbf{.007 $\pm$ .004} & .806 $\pm$ .011 & .052 $\pm$ .004 \\
    RW+AD & .808 $\pm$ .010 & .025 $\pm$ .007 & .800 $\pm$ .015 & .035 $\pm$ .003 \\
    \textbf{Ours}+AD & .803 $\pm$ .009 & \textbf{.007 $\pm$ .003} & .805 $\pm$ .009 & \textbf{.033 $\pm$ .003} \\
    \cmidrule(l){1-5}
    FB & .825 $\pm$ .002 & .017 $\pm$ .001 & .826 $\pm$ .001 & .049 $\pm$ .001 \\
    RW+FB & .818 $\pm$ .008 & .020 $\pm$ .005 & .817 $\pm$ .008 & .071 $\pm$ .010 \\
    \textbf{Ours}+FB & .824 $\pm$ .001 & \textbf{.008 $\pm$ .003} & .826 $\pm$ .003 & \textbf{.037 $\pm$ .012} \\
    \bottomrule
  \end{tabular}
  }
\end{table}




\vspace{0.5cm}
\subsection{\yuji{How the Pre-processed Data Aligns with the True Test Data}}
\label{appendix:align_well}

\yuji{Continuing from Sec.~\ref{sec:accfair}, we also show that the pre-processed data distribution ($D_\text{pre}$) by our algorithm is more aligned with the true test distribution ($D_\text{test}$) compared to the original training distribution ($D_\text{train}$), in terms of ($y$, $z$)-correlation and Wasserstein distance. 
We calculate the second-order Wasserstein distance via an optimal transport technique~\citep{peyre2019computational} on the synthetic data. We experiment on three degrees on shifts and make two observations in Table~\ref{tbl:alignment}: 1) the correlations of $D_\text{pre}$ and $D_\text{test}$ are indeed more similar relative to $D_\text{train}$ and 2) the Wasserstein distance between $D_\text{pre}$ and $D_\text{test}$ is lower than between $D_\text{train}$ and $D_\text{test}$. Both observations confirm that the reweighed data aligns well with the test data.}

\yuji{As our method improves the alignment between $D_\text{pre}$ and $D_\text{test}$, the in-processing algorithms trained on $D_\text{pre}$ (i.e., ours + in-processing algorithms) show high accuracy and fairness performances on $D_\text{test}$, as shown in Sec.~\ref{sec:experiments}.}

\begin{table}[h]
  \caption{\yuji{Alignment between data distributions. We use the synthetic data.}}
  \label{tbl:alignment}
  \centering
\scalebox{0.9}{
  \begin{tabular}{l@{\hspace{10pt}}c@{\hspace{20pt}}c@{\hspace{7pt}}c}
    \toprule
      & \multicolumn{3}{c}{Level of correlation shift in $c_\text{test}$} \\
    \midrule
    & Severe & Normal & No shift \\
    \midrule
     $c_\text{train}$ (correlation in $D_\text{train}$) & 0.3591 & 0.3591 & 0.3591 \\
     $c_\text{pre}$ (correlation in $D_\text{pre}$) & 0.0359 & 0.1796 & 0.3590 \\
     $c_\text{test}$  (correlation in $D_\text{test}$) & 0.0360 & 0.1800 & 0.3591 \\
    \cmidrule(l){1-4}
    Wass. dist. between $D_\text{train}$ and $D_\text{test}$ & 0.7966 & 0.6807 & 0.4399 \\
    Wass. dist. between $D_\text{pre}$ and $D_\text{test}$ & 0.6257 & 0.6022 & 0.4399 \\
    \bottomrule
  \end{tabular}
  }
\end{table}

\subsection{Accuracy and Fairness -- Other Test Data Construction: \textcolor{black}{Synthetic Data}}
\label{appendix:test_construction}

Continuing from Sec.~\ref{sec:accfair}, we compare the algorithm performances when using a different method to construct the test dataset \textcolor{black}{of the synthetic-data experiment}. Table~\ref{tbl:synthetic_corr2} shows the accuracy and fairness performances on the synthetic dataset when constructing the test dataset by modifying $\rz$ directly. 

We again generate the synthetic dataset using a method similar to \citet{DBLP:conf/aistats/ZafarVGG17}. The synthetic dataset consists of two non-sensitive attributes $(\rx_1, \rx_2)$, one sensitive attribute $\rz$, and one label attribute $\ry$. Each sample $(\rx_1, \rx_2, \ry)$ is drawn from the following Gaussian distributions: $(\rx_1,\rx_2)|\ry=0 \sim \mathcal{N}([-2; -2], [10, 1; 1, 3])$ and $(\rx_1,\rx_2)|\ry=1 \sim \mathcal{N}([2; 2], [5, 1; 1, 5])$. For the sensitive attribute $\rz$, we generate a biased distribution: $\Pr(\rz=1)=\Pr((\rx'_1, \rx'_2)|\ry=1)/[\Pr((\rx'_1, \rx'_2)|\ry=0)+\Pr((\rx'_1, \rx'_2)|\ry=1)]$ where $(\rx'_1, \rx'_2)=(\rx_1\cos(\pi/k)-\rx_2\sin(\pi/k), \rx_1\sin(\pi/k)+\rx_2\cos(\pi/k))$.
For the training dataset, we set $k$ to 4. 
We then change the $k$ value to generate the test dataset with 50\% of the training data correlation.

As a result, we observe consistent results where our framework improves the accuracy and fairness performances of the in-processing-only baselines and generally shows better fairness than the two-step baselines using RW.

\begin{table}[t]
  \caption{Performances on the synthetic test dataset. The test dataset is constructed via modifying the $\rz$ values of the original distribution. Other settings are identical to those in Table~\ref{tbl:synthetic}.}
  \label{tbl:synthetic_corr2}
  \centering
\scalebox{0.9}{
  \begin{tabular}{l@{\hspace{7pt}}c@{\hspace{7pt}}c@{\hspace{12pt}}c@{\hspace{7pt}}c}
    \toprule
      & \multicolumn{2}{c}{Single (DP)} & \multicolumn{2}{c}{Multiple (DP \& EO)} \\
    \cmidrule(r){1-5}
      Method & Acc. & Unfair. & Acc. & Unfair. \\
    \midrule
    LR & .871 $\pm$ .000 & .138 $\pm$ .000 & .871 $\pm$ .000 & .138 $\pm$ .000 \\
    \cmidrule(l){1-5}
    FC & .805 $\pm$ .006 & .040 $\pm$ .005 & .830 $\pm$ .001 & .119 $\pm$ .003 \\
    RW+FC & .847 $\pm$ .005 & .047 $\pm$ .007 & .852 $\pm$ .005 & .053 $\pm$ .004 \\
    \textbf{Ours}+FC & .831 $\pm$ .003 & \textbf{.005 $\pm$ .004} & .854 $\pm$ .004 & \textbf{.052 $\pm$ .002} \\
    \cmidrule(l){1-5}
    AD & .792 $\pm$ .012 & .048 $\pm$ .009 & .815 $\pm$ .008 & .133 $\pm$ .006 \\
    RW+AD & .836 $\pm$ .010 & .042 $\pm$ .008 & .849 $\pm$ .004 & \textbf{.054 $\pm$ .004} \\
    \textbf{Ours}+AD & .820 $\pm$ .006 & \textbf{.005 $\pm$ .003} & .847 $\pm$ .002 & .057 $\pm$ .007 \\
    \cmidrule(l){1-5}
    FB & .804 $\pm$ .001 & .051 $\pm$ .002 & .831 $\pm$ .001 & .128 $\pm$ .002 \\
    RW+FB & .844 $\pm$ .004 & .036 $\pm$ .003 & .853 $\pm$ .005 & .057 $\pm$ .011 \\
    \textbf{Ours}+FB & .824 $\pm$ .002 & \textbf{.015 $\pm$ .002} & .854 $\pm$ .003 & \textbf{.054 $\pm$ .003} \\
    \bottomrule
  \end{tabular}
  }
  \vspace{-0.2cm}
\end{table}


\subsection{\textcolor{black}{Accuracy and Fairness -- Tradeoff Curves}}
\label{appendix:tradeoff}

\textcolor{black}{Continuing from Sec.~\ref{sec:accfair}, we provide the accuracy-fairness disparity (DP) trade-off curves of FairBatch (FB) and Ours+FB on the the synthetic data in Figure~\ref{fig:tradeoff_synthetic}.
Our framework shows a better accuracy-fairness tradeoff compared to the in-processing approach (FB). FB shows an ``inversed'' curve, as it is trained with wrong biases in mind.}

\begin{figure}[h]
\centering
\includegraphics[width=0.48\columnwidth,trim=0cm 0.7cm 0cm 0cm]{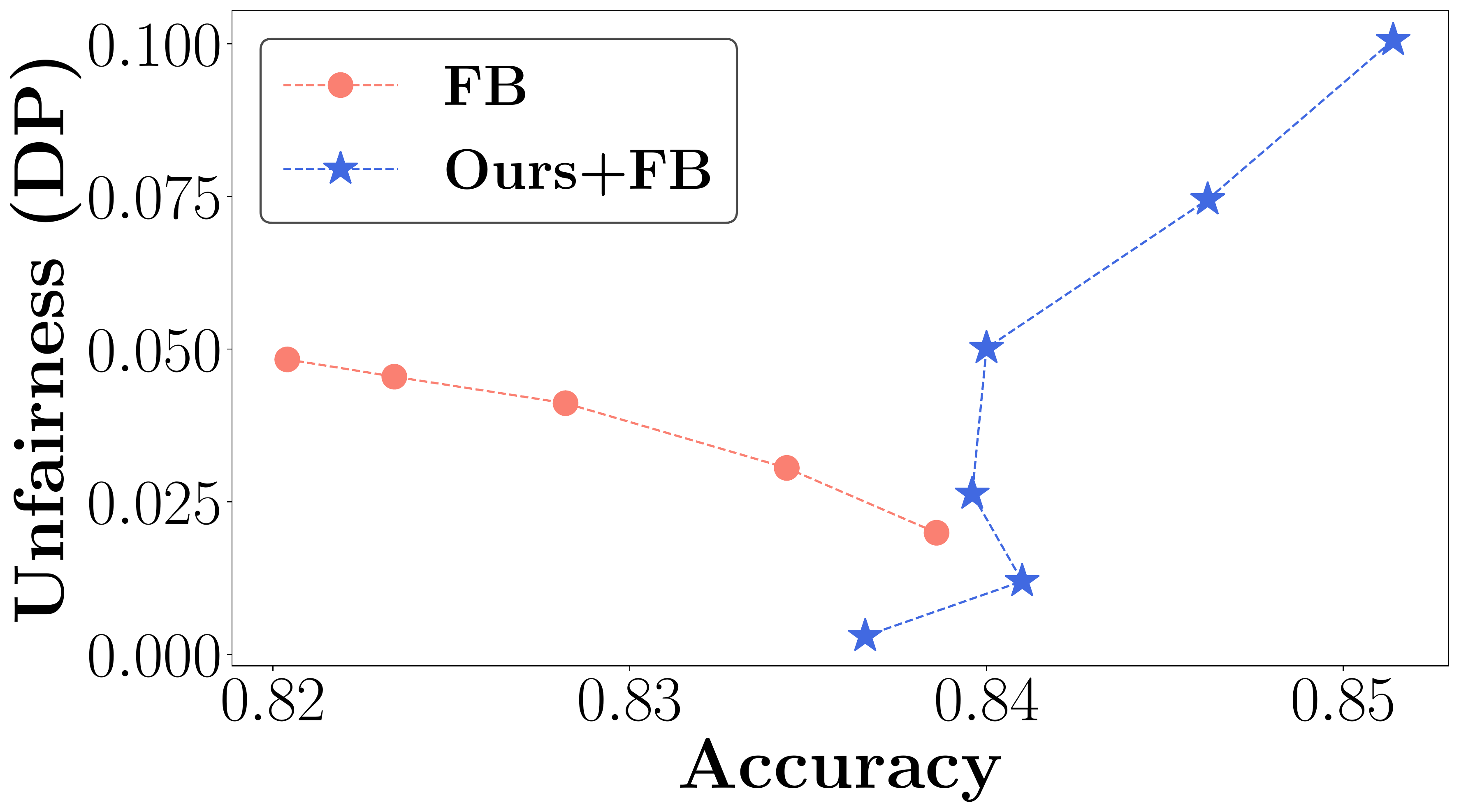}
\caption{Accuracy-fairness disparity trade-off curves of FB and Ours+FB on the synthetic data.}
\label{fig:tradeoff_synthetic}
\end{figure}

\subsection{Using the Optional Step}
\label{appendix:optional_step}

\revision{Continuing from Sec.~\ref{sec:accfair},} we show the results of when using the optional step using $\texttt{MinDistChange}$ in Algorithm~\ref{alg:overall}.
\textcolor{black}{This step finds possibly-different data sample weights within each ($\ry, \rz$)-class to minimize the overall distribution change between the training and pre-processed data.
Figure~\ref{fig:min_dist_change} shows the Wasserstein distances between the original training data and pre-processed data when our algorithm uses either the basic or optional step, and the correlation constant $c$ ranges from 10\% to 70\% of the training data's correlation. 
As a result, the optional step generally reduces the Wasserstein distance between the training and pre-processed data distributions, especially for smaller $c$ values.}
In addition, Table~\ref{tbl:synthetic_optional_step} shows the performances on the synthetic test dataset when our algorithm uses either the basic step (Ours) or the optional step (Ours+Optional).
As a result, when using the new data from the optional step as an input of the in-processing approaches, the final classifier's accuracy and fairness are not sacrificed much compared to using data from the basic step, while minimally changing the data distribution.

\begin{figure}[h]
\centering
\includegraphics[width=0.5\columnwidth,trim=0cm 0.3cm 0cm 0cm]{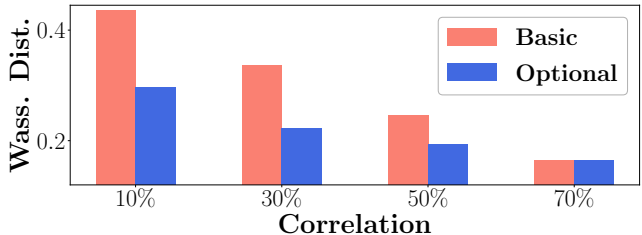}
\vspace{-0.4cm}
\caption{Wasserstein distances between the original training data and pre-processed data when our algorithm uses either the basic step or the optional step, while varying the target correlation to have 10\% to 70\% correlation of the training data.}
\label{fig:min_dist_change}
\end{figure}

\begin{table}[h]
\vspace{-0.5cm}
  \caption{Performances on the synthetic test dataset. Other settings are identical to those in Table~\ref{tbl:synthetic}.}
  \label{tbl:synthetic_optional_step}
  \centering
\scalebox{0.9}{
  \begin{tabular}{l@{\hspace{7pt}}c@{\hspace{7pt}}c@{\hspace{12pt}}}
    \toprule
      & \multicolumn{2}{c}{Single (DP)} \\
    \cmidrule(r){1-3}
      Method & Acc. & Unfair. \\
    \midrule
    \textbf{Ours}+FC & .849 $\pm$ .002 & .034 $\pm$ .004 \\
    \textbf{Ours+Optional}+FC & .830 $\pm$ .003 & .027 $\pm$ .005  \\
    \cmidrule(l){1-3}
    \textbf{Ours}+AD & .814 $\pm$ .011 & .017 $\pm$ .006  \\
    \textbf{Ours+Optional}+AD & .808 $\pm$ .007 & .014 $\pm$ .006  \\
    \cmidrule(l){1-3}
    \textbf{Ours}+FB & .836 $\pm$ .001 & .003 $\pm$ .001  \\
    \textbf{Ours+Optional}+FB & .842 $\pm$ .002 & .010 $\pm$ .004  \\
    \bottomrule
  \end{tabular}
  }
\end{table}




\vspace{-0.1cm}
\subsection{Varying the Correlation of the Test Data: More Biases}
\label{appendix:varying_corr_larger}
Continuing from Sec.~\ref{sec:varying_corr}, we vary the correlation of the test data to have more biases than the training data. Figure~\ref{fig:varying_corr_all} shows the accuracy and fairness performances of FB and Ours+FB, while varying the correlation constant $c$ of the test data up to 150\%. When the correlation increases more than 100\% of that of the training data, the in-processing-only baseline (FB) cannot improve fairness because the training data does not capture the bias level in the test data. Hence, the in-processing-only baseline cannot be used in applications that require high fairness. On the other hand, our pre-processing enables the in-processing approach to achieve the high fairness that it may need, with some accuracy degradation.

\begin{figure}[h]
\centering
\includegraphics[width=0.78\columnwidth,trim=0cm 0.3cm 0cm 0cm]{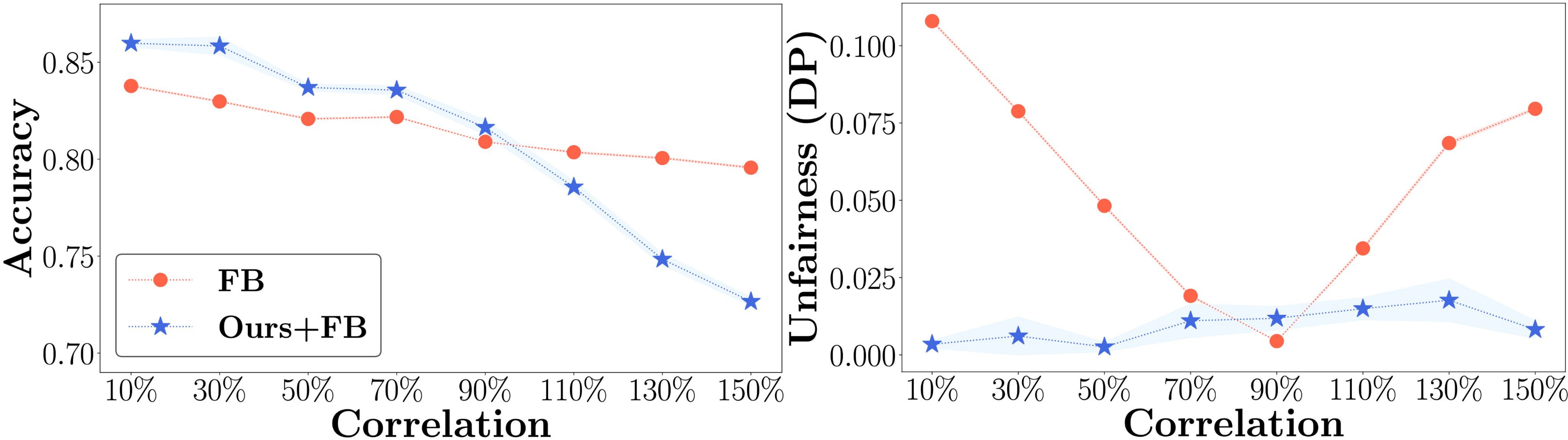}
\vspace{-0.2cm}
\caption{Performances of FB and Ours+FB on the synthetic data while varying the correlation of the test data to have 10\% to 150\% correlation of the training data.}
\label{fig:varying_corr_all}
\end{figure}

\subsection{Accuracy and Fairness -- Other Test Data Construction: Real Data}
\label{appendix:newadult}

Continuing from Sec.~\ref{sec:realworld}, we compare the algorithm performances when using a different method to construct the test dataset of the real-data experiment. Table~\ref{tbl:new_adult} shows the accuracy and fairness performances when using the two income datasets collected in the 1990s~\citep{DBLP:conf/kdd/Kohavi96} (i.e., AdultCensus dataset used in Sec.~\ref{appendix:adultcensus}) and 2010s~\citep{ding2021retiring} (i.e., ACSIncome dataset) for training and testing, respectively. Here, the 1990s and 2010s datasets have the different ($\ry$, $\rz$)-correlation values.
\revision{In addition, the input feature $\rx$ distribution also slightly shifted, as shown in Figure~\ref{fig:pca_graph}.}

\revisionnew{In this real-world scenario, we first estimate the shifted correlation constant $c$ and range $[\alpha, \beta]$ by accessing some of the 2010s (deployment) data. Here, we show how the number of samples affects the shift range $[\alpha, \beta]$ using Theorem~\ref{thm:correlation}. Figure~\ref{fig:correlation_estimation} shows the estimated correlation values when varying the number of samples used in the estimation. The blue dots indicate the estimated correlation constant $c$ values that are inferred by the maximum likelihood estimation (MLE). The bars show the error ranges of each estimation with 90\% probability, which can be calculated using Theorem~\ref{thm:correlation} -- see details in Sec.~\ref{appendix:estimate_corr}. 
As discussed in Theorem~\ref{thm:correlation}, a larger number of samples yields a smaller range. In addition, although the range is large when the number of samples decreases (e.g., $m{=}100$), the point estimate value for $c$ is close to the true value.
}

\revisionnew{We then run our pre-processing by setting $[\alpha, \beta]$ with the error range of the estimation. We choose the values when the number of samples is 1,000 ($\simeq$ 0.07\%), as it has a reasonable error range with a suitable number of samples.}
As a result, we observe consistent results in Table~\ref{tbl:new_adult} where our framework improves the accuracy and fairness performances of the in-processing-only baselines.


\begin{figure}[h]
\centering
\includegraphics[width=0.5\columnwidth,trim=0cm 0.7cm 0cm 0cm]{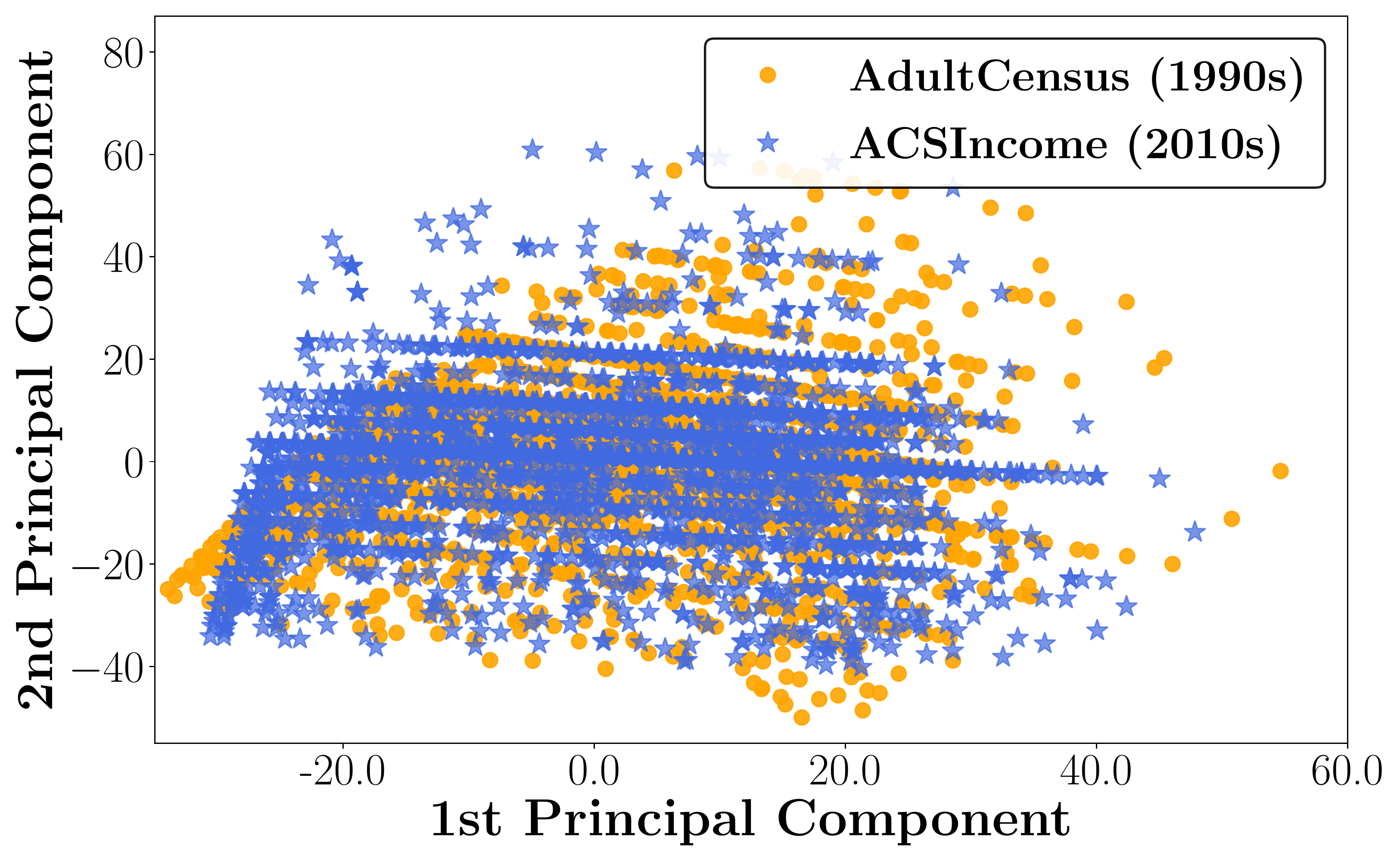}
\caption{\revision{Principal component analysis (PCA) results on the input feature $\rx$ of the two US income datasets: AdultCensus~\citep{DBLP:conf/kdd/Kohavi96} and ACSIncome~\citep{ding2021retiring}. Here, the distributions of the principal components on input feature $\rx$ are different in the two datasets.}}
\vspace{0.2cm}
\label{fig:pca_graph}
\end{figure}

\begin{figure}[h]
\centering
\includegraphics[width=0.5\columnwidth,trim=0cm 0.7cm 0cm 0cm]{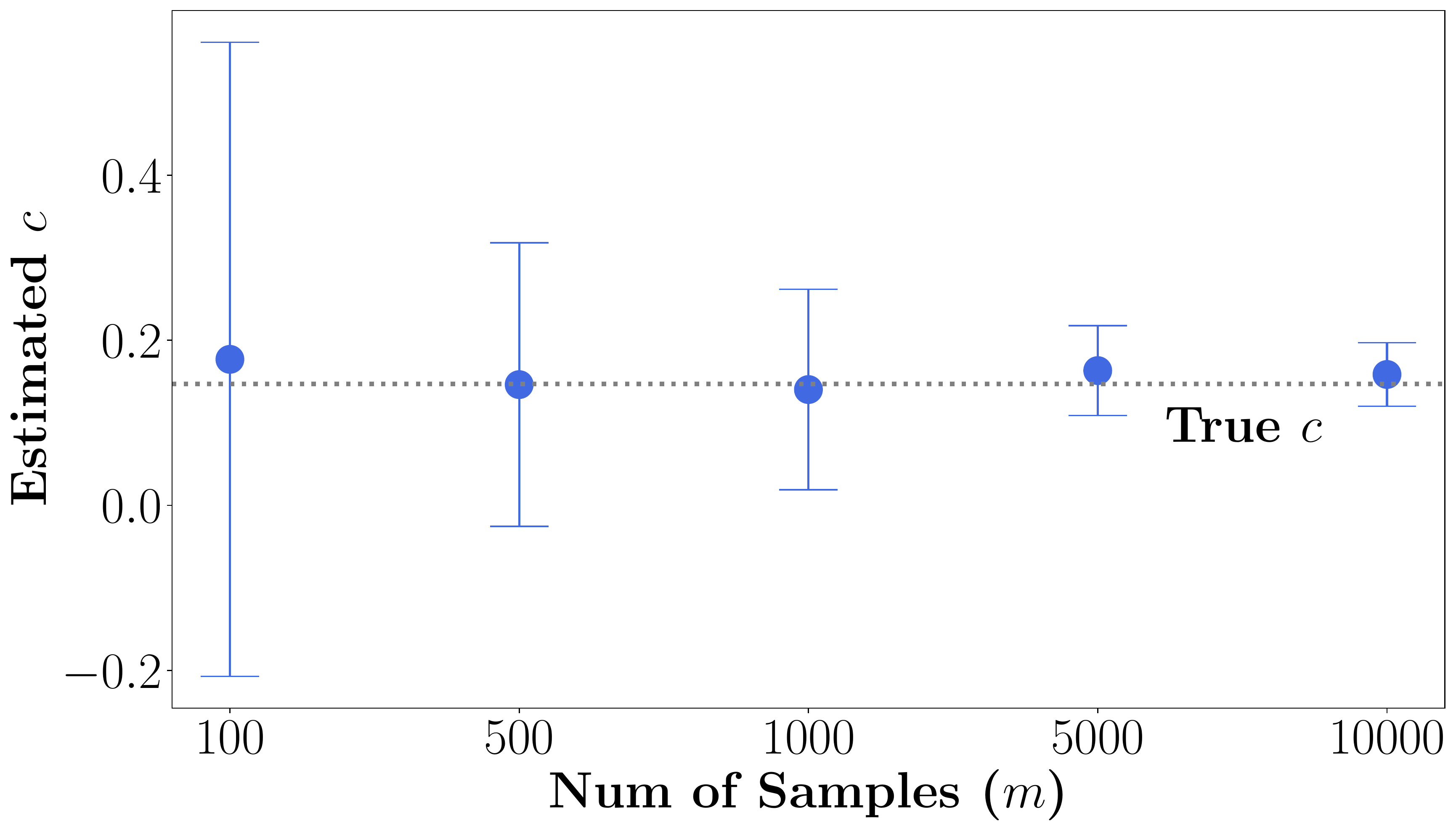}
\caption{\revisionnew{Estimating the correlation of the ACSIncome (deployment) data when varying the number of samples used in the estimation. The blue dots indicate the estimated correlation constant $c$ values that are inferred by the maximum likelihood estimation (MLE). The bars indicate the error ranges of each estimation with 90\% probability, which can be calculated using Theorem~\ref{thm:correlation}. The error ranges can be used as $[\alpha, \beta]$ in our pre-processing step. Here, the true $c$ of ACSIncome is 0.147.}}
\vspace{0.2cm}
\label{fig:correlation_estimation}
\end{figure}

\begin{table}[h]
  \caption{Model performances on the ACSIncome test data ~\citep{ding2021retiring} collected in the 2010s. The models are trained on the AdultCensus data~\citep{DBLP:conf/kdd/Kohavi96} collected in the 1990s. Both datasets have labels that indicate each person's annual income. \revision{In our pre-processing, we use the estimated correlation range $[\alpha, \beta]$ that is inferred by using 1,000 ($\simeq$ 0.07\%) samples of the ACSIncome data. 
  }}
  \label{tbl:new_adult}
  \centering
\scalebox{0.87}{
  \begin{tabular}{l@{\hspace{7pt}}c@{\hspace{7pt}}c@{\hspace{12pt}}}
    \toprule
      & \multicolumn{2}{c}{Single (DP)} \\
    \cmidrule(r){1-3}
      Method & Acc. & Unfair. \\
    \midrule
    FC & .646 $\pm$ .013 & .104 $\pm$ .015 \\
    \textbf{Ours}+FC & .665 $\pm$ .010 & .079 $\pm$ .011  \\
    \cmidrule(l){1-3}
    AD & .637 $\pm$ .031 & .137 $\pm$ .037  \\
    \textbf{Ours}+AD & .638 $\pm$ .023 & .051 $\pm$ .048  \\
    \cmidrule(l){1-3}
    FB & .643 $\pm$ .007 & .128 $\pm$ .021  \\
    \textbf{Ours}+FB & .694 $\pm$ .002 & .068 $\pm$ .006  \\
    \cmidrule(l){1-3}
    {\em FB on test dist. (upper bound)} & .704 $\pm$ .007 & .044 $\pm$ .035 \\
    \bottomrule
  \end{tabular}
  }
\end{table}

\vspace{-0.1cm}
\subsection{Setting the Correlation Range to $c_\text{test}\pm x\%$}
\label{appendix:c_range}

Continuing from Sec.~\ref{sec:misspecifying_corr}, we evaluate our approach when $[\alpha, \beta] = [c_\text{test}-x\%, c_\text{test}+x\%]$. Figure~\ref{fig:correlation_range} shows the accuracy and fairness performances of FB and Ours+FB, while varying the correlation constant $c$ of the test data. Here, we set the \textit{specified} correlation range in our algorithm to \textcolor{black}{$[\alpha, \beta] = [c_\text{test}-x\%, c_\text{test}+x\%]$, where $x \in \{10, 50, 100\}$}. 
\textcolor{black}{When $x = 10\%$,} we observe similar trends as in Figure~\ref{fig:varying_corr}, where our framework improves the accuracy and fairness performances of the in-processing-only baselines.
\textcolor{black}{When $x = 100\%$ (i.e., the worst-case setting of $[\alpha, \beta]$), the accuracy and fairness performances of our framework converge to the in-processing-only baselines. 
When $x = 50\%$, the performances of our framework are in between those of when $x = 10\%$ and $x = 100\%$. 
There are two takeaways: 1) our framework successfully boosts the in-processing-only baseline performances when the $[\alpha, \beta]$ range is reasonable, and 2) even if we do not have any information about the correlation shift, our framework performs at least as well as the in-processing-only baselines.}

\begin{figure}[h]
\vspace{0.4cm}
\centering
\includegraphics[width=0.8\columnwidth,trim=0cm 0.3cm 0cm 0cm]{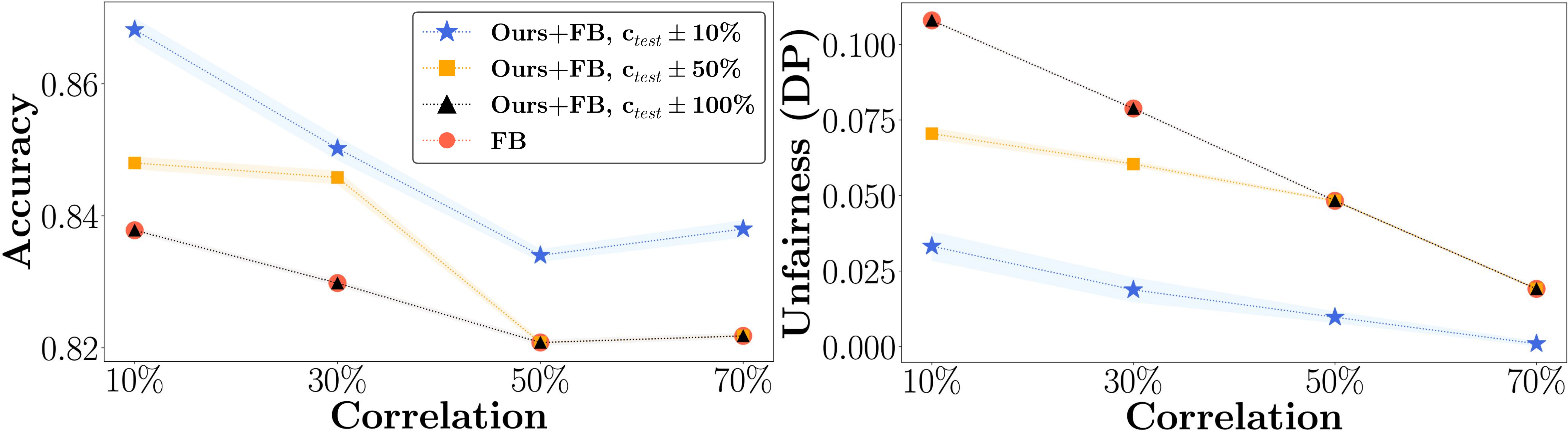}
\caption{Performances of FB and Ours+FB on the synthetic data while varying the correlation of the test data $c_\text{test}$ to have 10\% to 70\% correlation of the training data. We set the \textit{specified} correlation range in our algorithm to $[\alpha, \beta] = [c_\text{test}-x\%, c_\text{test}+x\%]$, where $x \in \{10, 50, 100\}$.}
\label{fig:correlation_range}
\end{figure}

\subsection{\revision{Robust Against Misspecified Correlations: Real Data}}
\label{appendix:unknown_corr_compas}

\revision{Continuing from Sec.~\ref{sec:misspecifying_corr}, we evaluate our approach when misspecifying the correlation in our algorithm on the real-world COMPAS dataset~\citep{Compas}. Same with Sec.~\ref{sec:misspecifying_corr}, we set $\alpha = \beta = c_\text{specified}$, where $c_\text{specified} \neq c_\text{test}$. Figure~\ref{fig:wrong_corr_compas} shows the accuracy and fairness performances of FB and Ours+FB when the true correlation of the test data is 60\% of the training data's correlation. While achieving similar or higher accuracies compared to the in-processing-only baseline for the entire range of considered correlations, our framework improves the fairness of the in-processing-only baseline (FB) when the specified correlation is higher than 30\%.} 
\revision{This result shows that our approach is still beneficial when the estimation error is reasonable (say $\pm$10\%).}
\revisionnew{We also note that when the specified correlation is far from the true value, our framework may not improve the performances of the baseline.}

\begin{figure}[h]
\centering
\includegraphics[width=0.8\columnwidth,trim=0cm 0.9cm 0cm 0cm]{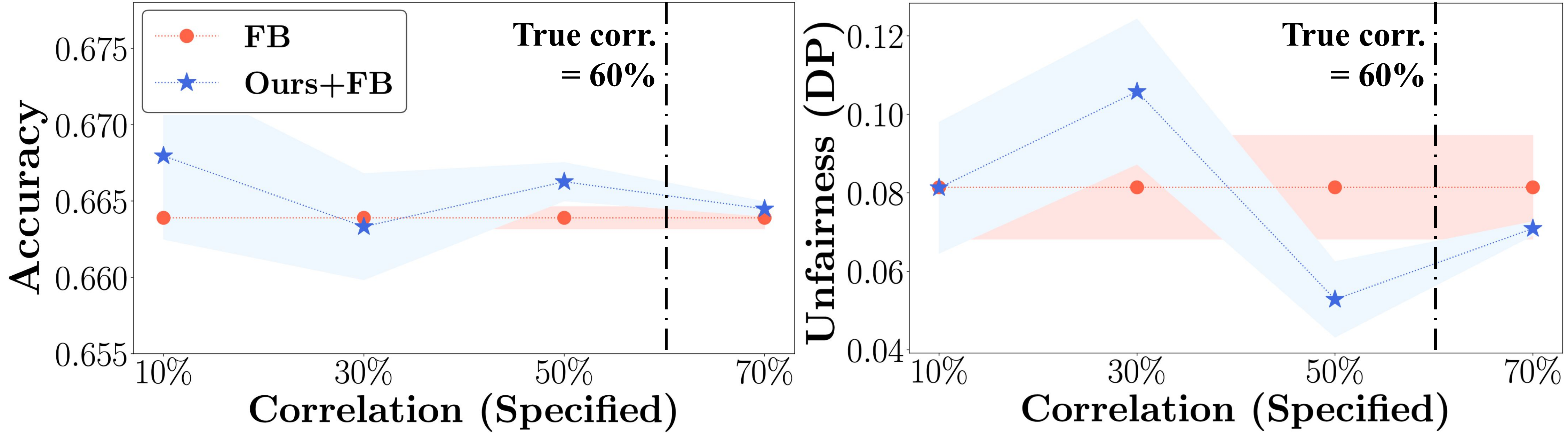}
\caption{\small \revision{Performances of FB and Ours+FB on the COMPAS data while varying the \textit{specified} correlation in our algorithm to have 10\% to 70\% where the true correlation of the test data is 60\%.}}
\label{fig:wrong_corr_compas}
\end{figure}

\subsection{\revision{Empirical Comparison with \citet{giguere2022fairness}}}
\label{appendix:shifty}

\revision{Continuing from Sec.~\ref{sec:relatedwork}, we add a new baseline called Shifty~\citep{giguere2022fairness}, which focuses on the distribution shift most relevant to ours. Shifty first trains candidate models on the training data and selects only the models showing high fairness in the shifted deployment data. 
To give a favorable condition to Shifty, we assume that Shifty knows the exact test distribution. Table~\ref{tbl:shifty} shows the accuracy and fairness performances of the in-processing-only baseline FairBatch, Shifty, and our framework w.r.t.\@ a single metric (DP) and multiple metrics (DP \& EO) in the synthetic and COMPAS datasets used in Tables~\ref{tbl:synthetic} and~\ref{tbl:compas}. 
As a result, both ours and Shifty improve the fairness of the in-processing-only baseline, but ours shows better fairness than Shifty while achieving similar or higher accuracy.
The reason is that Shifty is selecting the final model among the candidates that were already trained on the original training data, whereas ours trains a new model on the improved (pre-processed) data.}


\begin{table}[h]
\vspace{-0.1cm}
  \caption{\revision{Performances on the synthetic and COMPAS test datasets. We compare in-processing-only baseline FairBatch~\citep{roh2021fairbatch}, Shifty~\citep{giguere2022fairness}, and our framework. Other settings are identical to those in Table~\ref{tbl:synthetic}.}}
  \label{tbl:shifty}
  \centering
\scalebox{0.85}{
  \begin{tabular}{l@{\hspace{12pt}}c@{\hspace{7pt}}c@{\hspace{12pt}}c@{\hspace{7pt}}c@{\hspace{15pt}}c@{\hspace{7pt}}c@{\hspace{12pt}}c@{\hspace{7pt}}c}
    \toprule
    & \multicolumn{4}{c}{Synthetic} & \multicolumn{4}{c}{COMPAS}\\
    \cmidrule(r){1-9}
      & \multicolumn{2}{c}{Single (DP)} & \multicolumn{2}{c}{Multiple (DP \& EO)} & \multicolumn{2}{c}{Single (DP)} & \multicolumn{2}{c}{Multiple (DP \& EO)}\\
    \cmidrule(r){1-9}
      Method & Acc. & Unfair. & Acc. & Unfair. & Acc. & Unfair. & Acc. & Unfair. \\
    \midrule
    FB~\citep{roh2021fairbatch} & .821 $\pm$ .000 & .048 $\pm$ .000 & .849 $\pm$ .001 & .091 $\pm$ .005 & .647 $\pm$ .001 & .038 $\pm$ .013 & .650 $\pm$ .002 & .187 $\pm$ .019 \\
    FB + Shifty & \textbf{.838 $\pm$ .001} & .024 $\pm$ .008 & .844 $\pm$ .003 & .063 $\pm$ .008 & .647 $\pm$ .001 & .029 $\pm$ .001 & .649 $\pm$ .002 & .162 $\pm$ .025\\
    \textbf{Ours}+FB & .836 $\pm$ .001 & \textbf{.003 $\pm$ .001} & \textbf{.852 $\pm$ .004} & \textbf{.058 $\pm$ .001} & \textbf{.648 $\pm$ .004} & \textbf{.027 $\pm$ .001} & \textbf{.657 $\pm$ .004} & \textbf{.130 $\pm$ .014}\\
    \bottomrule
  \end{tabular}
  }
\end{table}

\subsection{\yuji{Supporting Noisy Group Attributes}}
\label{appendix:noisy}

\yuji{Continuing from Sec.~\ref{sec:relatedwork}, we evaluate the performance of our method on increasingly noisy data and observe that our method can potentially be extended to the noisy group scenario.}

\yuji{We perform a new experiment, where 10--50\% of the group information in the training data is randomly flipped. Table~\ref{tbl:noisy_group} shows the accuracy and fairness performances of the in-processing-only baseline FairBatch (FB) and our framework (Ours+FB) w.r.t.\@ demographic parity (DP). Here, FB cannot achieve high fairness performance (i.e., low DP disp.), as the group distribution in the training data is different from that in the test data. In comparison, our framework enables FB to achieve relatively high fairness at the expense of some accuracy degradation.}

\yuji{We note that these results are preliminary and only show that our method has some potential to support noisy group attributes, and we believe there is plenty of room for improvement. In particular, we believe our method can be further extended with other robust training methods like \citep{pmlr-v97-shen19e}.}

\begin{table}[h]
\vspace{-0.2cm}
  \caption{\yuji{Performances on the synthetic test dataset when the group information in the training data is randomly flipped. We train the algorithms w.r.t.\@ demographic parity (DP).}}
  \label{tbl:noisy_group}
  \centering
\scalebox{0.9}{
  \begin{tabular}{l@{\hspace{12pt}}c@{\hspace{12pt}}c@{\hspace{12pt}}c@{\hspace{12pt}}c@{\hspace{12pt}}c@{\hspace{12pt}}c@{\hspace{12pt}}}
    \toprule
      & \multicolumn{2}{c}{10\% flipping} & \multicolumn{2}{c}{30\% flipping} & \multicolumn{2}{c}{50\% flipping} \\
    \cmidrule(r){2-3}\cmidrule(r){4-5}\cmidrule(r){6-7}
      Method & Acc. & \breakcell{Unfair.\\(DP)} & Acc. & \breakcell{Unfair.\\(DP)} & Acc. & \breakcell{Unfair.\\(DP)} \\
    \cmidrule(r){1-1}\cmidrule(r){2-3}\cmidrule(r){4-5}\cmidrule(r){6-7}
    FB & 0.818 & 0.087 & 0.864 & 0.165 & 0.883 & 0.250 \\
    \textbf{Ours} + FB & 0.808 & 0.069 & 0.823 & 0.101 & 0.838 & 0.142 \\
    \bottomrule
  \end{tabular}
  }
\end{table}

\subsection{\revision{Supporting Poisoning Attack Scenario}}
\label{appendix:poisoning}

\revision{Continuing from Sec.~\ref{sec:relatedwork}, we evaluate the performance of our method on poisoned data and observe that our method can potentially be extended to the poisoning attack scenario.}

\revision{We perform a new experiment using one of the poisoning attack methods for fair training~\citep{pmlr-v119-roh20a}, where 10\% of the training labels of a specific group are flipped so as to maximize the accuracy degradation. Table~\ref{tbl:poisoning} shows the accuracy and fairness performances of the in-processing-only baseline FairBatch (FB) and our framework w.r.t.\@ demographic parity (DP) on the poisoned data; we also report FB's performance on the clean data, which can be considered as the upper-bound performance. Here, when achieving similar fairness, our framework improves the accuracy of FB on the poisoned data. We suspect that, while our pre-processing reduces the correlation shift between training and deployment data, the effect of the poisoning attack is also mitigated.}

\revision{We note that these results are preliminary and only show that our method has some potential to support the poisoning attack scenario, and we believe there is plenty of room for improvement.}

\begin{table}[h]
\vspace{-0.2cm}
  \caption{\revision{Performances on the synthetic test dataset when 10\% of the training labels are poisoned~\citep{pmlr-v119-roh20a}. We train the algorithms w.r.t.\@ demographic parity (DP).}}
  \label{tbl:poisoning}
  \centering
\scalebox{0.9}{
  \begin{tabular}{l@{\hspace{12pt}}c@{\hspace{12pt}}c}
    \toprule
      Method & Acc. & \breakcell{Unfair. (DP)} \\
    \cmidrule(l){1-3}
    FB & 0.699 & 0.026 \\
    \textbf{Ours} + FB & 0.721 & 0.024 \\
    \cmidrule(l){1-3}
    {\em FB on the clean data (upper bound)} & 0.773 & 0.023 \\
    \bottomrule
  \end{tabular}
  }
\end{table}

\section{Appendix -- More Related Work}
\label{appendix:related_work}

\vspace{0.2cm}
\subsection{\revision{Traditional Model Fairness}}
\label{appendix:relatedwork_traditional_fairness}

As model fairness becomes essential for Trustworthy AI, various fairness definitions have been proposed to reflect legal and social requirements~\citep{narayanan2018translation}. 
Among the definitions, we focus on group fairness, which aims to not discriminate specific groups. There are three prominent group fairness metrics: demographic parity~\citep{DBLP:conf/kdd/FeldmanFMSV15}, equalized odds~\citep{DBLP:conf/nips/HardtPNS16}, and predictive parity~\citep{berk2021fairness}.
To support the fairness metrics, various fairness techniques have been proposed, which can be categorized into three prominent approaches: (1) pre-processing approaches~\citep{DBLP:journals/kais/KamiranC11, DBLP:conf/icml/ZemelWSPD13, DBLP:conf/kdd/FeldmanFMSV15, DBLP:conf/nips/CalmonWVRV17, choi2020fair, pmlr-v108-jiang20a}, 
which debias, reweight, or generate training data, (2) in-processing approaches~\citep{DBLP:conf/aistats/ZafarVGG17, DBLP:conf/www/ZafarVGG17, DBLP:conf/icml/AgarwalBD0W18, DBLP:conf/aies/ZhangLM18, DBLP:conf/alt/CotterJS19, pmlr-v119-roh20a, roh2021fairbatch},
which modify model training itself for fairness, and (3) post-processing approaches~\citep{DBLP:conf/icdm/KamiranKZ12, DBLP:conf/nips/HardtPNS16, DBLP:conf/nips/PleissRWKW17, NIPS2019_9437},
which manipulate only the model outputs without changing the training inside.
Among the three categories, in-processing approaches are widely used for unfairness mitigation, but most of them assume that the training and deployment data distributions are the same.

Another line of research is to support multiple fairness metrics~\citep{thomas2019preventing, Zhao2020Conditional}. \citet{thomas2019preventing} proposes a fairness testing framework that can support multiple metrics. \citet{Zhao2020Conditional} shows that EO can be achieved while preserving the original DP. In comparison, we analyze when a model can achieve both $\varepsilon$-DP and $\varepsilon$-EO.

Beyond group fairness, there are other noteworthy fairness definitions including individual fairness~\citep{DBLP:conf/innovations/DworkHPRZ12}, which aims to give similar predictions to similar individuals, and causality-based fairness~\citep{10.5555/3294771.3294834, NIPS2017_6995, Zhang2018FairnessID}, which aims to improve fairness by understanding the causal relationship between attributes. Extending our analysis and framework to these definitions is an interesting future work.

\vspace{0.2cm}
\subsection{\revision{Fairness under Data Distribution Shifts}}
\label{appendix:relatedwork_distribution_shifts}

\revision{Continuing from Sec.~\ref{sec:relatedwork}, we further compare our work with the previous studies on data distribution shifts. The following paragraphs contain detailed discussions of the two categories of distribution shifts: general distribution shifts and fairness-specific shifts.}

\revision{Among the general distribution shifts (i.e., covariate, label, and concept shifts in Table~\ref{tbl:related_work}), the concept shift is the most relevant definition to the correlation shift. As shown in Table~\ref{tbl:related_work} of Sec.~\ref{sec:relatedwork}, the correlation and concept shifts focus on $\Pr(\rz|\ry)$ and $\Pr(\ry|\rx)$, respectively. As the concept shifts consider the distribution changes of the label ($\ry$) and input feature ($\rx$), this type of shift can implicitly describe the correlation shifts when the input feature contains the group attribute ($\rz$). However, a unique characteristic of the correlation shifts is to explicitly capture the bias changes between $\ry$ and $\rz$, where $\rz$ is especially relevant to fair training. Thus, the notion of correlation shifts enables us to analyze the behavior of fair training when the bias changes.
}

\revision{In addition, the key difference between the correlation shift and the other fairness-specific shifts (subpopulation and demographic shifts) is that the other shifts are defined under specific assumptions on the data distribution, which are not required in our correlation shift definition. Here are the assumptions in the subpopulation and demographic shifts:
\vspace{-0.1cm}
\begin{itemize}[topsep=0ex,partopsep=0ex,leftmargin=3mm,itemsep=0ex]
    \item The subpopulation shift~\citep{maity2021does} assumes that the loss expectations w.r.t. input feature X and label Y of training and deployment distributions are the same (i.e., $E_\text{train}[(h(\rx), \ry) | \rz=z] = E_\text{test}[(h(\rx), \ry) | \rz=z]$, where $E$ indicates the expectation). Thus, as discussed in the related work section, a major example of subpopulation shifts is when a specific group has fewer positively-labeled examples during training time compared to deployment time, while the distributions of the $\rx$ and $\ry$ attributes remain the same. \vspace{-0.1cm}
    \item Similarly, the demographic shift~\citep{giguere2022fairness} assumes that the joint probabilities of $\rx$ and $\ry$ on the training and deployment distributions are identical (i.e., $\Pr_\text{train}(\rx=x, \ry=y | \rz=z) = \Pr_\text{test}(\rx=x, \ry=y | \rz=z)$). This assumption is similar to the assumption of the subpopulation shift, but the difference is whether the loss values of the model are explicitly considered or not.
\end{itemize}
\vspace{-0.1cm}
The above assumptions are used in the theoretical analyses of the previous works~\citep{maity2021does, giguere2022fairness}. In comparison, our theoretical analyses in Sec.~\ref{sec:correlationshift} of using correlation shifts are not limited by any additional assumptions, and we show that the achievable performances of fair training are determined by the (y, z)-correlation.
}

\vspace{0.3cm}
\subsection{\revision{Connection to Causality-based Fairness}}
\label{appendix:causality}

\revision{Continuing from Sec.~\ref{sec:relatedwork}, we discuss how our work can be connected with causality-based fairness. To this end, we give a concrete example to show the correlation shifts when the sensitive group attribute itself is a confounder, one of the important roles of attributes in causality-based fairness.}

\revision{We consider a car insurance example, where the goal is to predict future accident rates from the past driving record: suppose that input feature $\rx$ is the past driving record, label $\ry$ is the future accident rate, and sensitive attribute $\rz$ is the race.}

\revision{Then, the structural causality graph between them would look like: $\rz$ $\rightarrow$ ($\rx$, $\ry$) and $\rx$ $\rightarrow$ $\ry$, i.e., $\rz$ being a confounder (Figure~\ref{fig:causal_graph}). The arrows from $\rz$ to $\rx$ and $\ry$ could be due to (a) different socialization and driving behavior formation process that partially depends on racial identity and (b) another unobserved mediator such as where they live in, etc. In this example where $\rz$ is a confounder, it is clear that $\rz$ and $\ry$ are correlated.}

\revision{We expect this confounding effect to change over time. This will make the sensitive attribute $\rz$ either a stronger or weaker confounding factor, which impacts the correlation between $\ry$ and $\rz$.
}

\begin{figure}[h]
\centering
\includegraphics[width=0.3\columnwidth,trim=0cm 0.3cm 0cm 0cm]{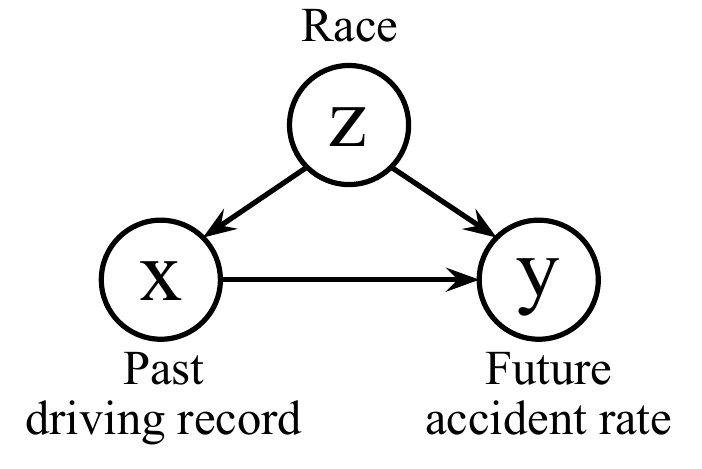}
\caption{\revision{A causality graph for the car insurance example in Sec.~\ref{appendix:causality}.}}
\label{fig:causal_graph}
\end{figure}

\end{document}

%% file: math_commands.tex
\usepackage{amsmath,amsfonts,bm}

\def\eqref#1{equation~\ref{#1}}

\def\1{\bm{1}}

\def\eps{{\epsilon}}

\def\rt{{\textnormal{t}}}

\def\rx{{\textnormal{x}}}
\def\ry{{\textnormal{y}}}
\def\rz{{\textnormal{z}}}

\DeclareMathAlphabet{\mathsfit}{\encodingdefault}{\sfdefault}{m}{sl}
\SetMathAlphabet{\mathsfit}{bold}{\encodingdefault}{\sfdefault}{bx}{n}

\def\sX{{\mathbb{X}}}
\def\sY{{\mathbb{Y}}}
\def\sZ{{\mathbb{Z}}}

\DeclareMathOperator{\diag}{diag}